\documentclass{article}

\usepackage[final]{corl_2020} %

\usepackage{amsmath,amssymb,amsthm}
\usepackage{amsfonts}
\usepackage{bm}
\usepackage{upgreek}
\usepackage{color}
\usepackage{multicol}
\usepackage{multirow} 
\usepackage{booktabs}
\usepackage{tabularx}
\usepackage{enumitem}
\usepackage{paralist}

\newcommand{\T}{\intercal}
\newcommand{\dx}{\delta_x}
\newcommand{\ddx}{\dot{\delta}_x}
\usepackage{scalerel,stackengine}
\stackMath
\newcommand\reallywidehat[1]{%
\savestack{\tmpbox}{\stretchto{%
  \scaleto{%
    \scalerel*[\widthof{\ensuremath{#1}}]{\kern-.6pt\bigwedge\kern-.6pt}%
    {\rule[-\textheight/2]{1ex}{\textheight}}%
  }{\textheight}%
}{0.5ex}}%
\stackon[1pt]{#1}{\tmpbox}%
}

\newtheorem{theorem}{Theorem}
\newtheorem{lemma}[theorem]{Lemma}

\newtheorem{proposition}[theorem]{Proposition}

\newcommand{\reals}{\mathbb{R}}
\newcommand{\X}{\mathcal{X}}
\newcommand{\U}{\mathcal{U}}
\newcommand{\rp}{\reals^{\geq 0}}

\newcommand{\toolname}{\textcolor{black}{C3M }}
\title{Learning Certified Control Using Contraction Metric}

\author{
  Dawei Sun\\
  SRI International and \\
  University of Illinois at Urbana-Champaign, US\\
  \texttt{daweis2@illinois.edu} \\
  \And
  Susmit Jha\\
  Computer Science Laboratory\\
  SRI International, US\\
  \texttt{susmit.jha@sri.com}
  \And
  Chuchu Fan\\
  Massachusetts Institute of Technology\\
  United States\\
  \texttt{chuchu@mit.edu} \\
}

\begin{document}

\maketitle

\begin{abstract}
In this paper, we solve the problem of finding a certified control policy that drives a robot from any given initial state and under any bounded disturbance to the desired reference trajectory, with guarantees on the convergence or bounds on the tracking error. Such a controller is crucial in safe motion planning. We leverage the advanced theory in Control Contraction Metric and design a learning framework based on neural networks to co-synthesize the contraction metric and the controller for control-affine systems. We further provide methods to validate the convergence and bounded error guarantees. We demonstrate the performance of our method using a suite of challenging robotic models, including models with learned dynamics as neural networks. We compare our approach with leading methods using sum-of-squares programming, reinforcement learning, and model predictive control. Results show that our methods indeed can handle a broader class of systems with less tracking error and faster execution speed. Code is available at \url{https://github.com/sundw2014/C3M}.
\end{abstract}


\section{Introduction}
Designing safe motion planning controllers for nonholonomic robotic systems is a critical yet extremely challenging problem. In general, motion planning for robots is notoriously difficult. 
For example, even planning for a robot (composed of $n$ polyhedral bodies)  in a 3D world that contains a fixed number of polyhedral obstacles is PSPACE-hard~\citep{lavalle2006planning}.
Simultaneously providing formal guarantees on safety and robustness when the robots are facing uncertainties and disturbances makes the problem even more challenging. Sample-based planning techniques~\citep{kavraki1998analysis,lavalle2006planning} can plan safe trajectories by exploring the environment but cannot handle uncertainty and disturbances. A natural thought that has been extensively explored is to exploit a separation of concerns that exists in the problem: (A) how to design a reference (also called expected or nominal) trajectory $x^*(t)$ to drive a robot to its goal safely without considering any uncertainty or disturbance, and (B) how to design a tracking controller to make sure the actual trajectories $x(t)$ of the system under disturbances can converge to $x^*(t)$ with guaranteed bounds for the tracking error between $x(t)$ and $x^*(t)$~\citep{herbert2017fastrack,vaskov2019towards,majumdar2017funnel,jha2018duality}. Combining controllers from solving (A) and (B) can make sure that the actual behaviors of the robot are all safe. 

While  the reference trajectories 
can be found using 
efficient planning techniques~\citep{kavraki1998analysis,lavalle2006planning,mpcPytorch} for a very broad class of systems, the synthesis of a guaranteed tracking controller, also called trajectory stabilization~\citep{singh2019learning}  is not an automatic process, especially for nonholonomic systems. Control theoretic techniques exist to guide the dual synthesis~\cite{jha2018duality} of control policies and certificates at the same time, where the certificates can make sure certain required properties are provably satisfied when the corresponding controller is run in (closed-loop) composition with the plant. For example, control Lyapunov Function (CLF)~\citep{chang2019neural, ravanbakhsh2016robust} ensures the existence of a controller so the controlled system is Lyapunov stable, and Control Barrier Function (CBF)~\citep{ames2014control,ames2019control,taylor2019learning} ensures that a controller exists so the controlled system always stays in certain safety invariant sets defined by the corresponding barrier function. Other methods pre-compute the tracking error through  reachability analysis~\citep{vaskov2019towards}, Funnels~\citep{tedrake2009lqr,majumdar2017funnel}, and Hamilton-Jacobi analysis~\citep{herbert2017fastrack,bansal2017safe}. 
Despite the benefit brought by 
these certificates (e.g. Lyapunov Function, Barrier Function), finding the correct function representation 
for the certificates is non-trivial. Various methods have been studied to learn the certificate as polynomials~\citep{kapinski2014simulation,ravanbakhsh2016robust}, support vectors~\citep{khansari2014learning}, 
Gaussian processes~\citep{jha2018data}, 
temporal logic formulae~\cite{jha2017telex,jha2018safe},
and neural networks (NN)~\citep{chang2019neural,taylor2019learning,taylor2019episodic,robey2020learning,choi2020reinforcement}. Unlike reinforcement learning (RL)~\cite{polydoros2017survey,ohnishi2019barrier,deisenroth2011pilco,chua2018deep} which focuses on learning a control policy that maximizes an accumulated reward and usually lacks formal safety guarantees, certificate-guided controller learning focuses on learning a sufficient condition for the desired property. In this paper, we follow this idea and learn a tracking controller by learning a convergence certificate simultaneously as guidance.

In this paper, we leverage recent advances in Control Contraction Metric (CCM) ~\citep{manchester2017control,manchester2015unifying} theory that extends the contraction theory to control-affine systems to prove the existence of a tracking controller so the closed-loop system is contracting. In the CCM theory~\citep{manchester2017control}, it has been shown that a valid CCM implies the existence of a contracting tracking controller for {\em any} reference trajectories. In our framework, we model both the controller and the metric function as neural networks and optimize them jointly with the cost functions inspired by the CCM theory. Due to the constraints imposed during training, the tracking error of the learned controller is guaranteed to be bounded even with external disturbances.

Synthesis of CCM certificate has been formulated as solving a Linear Matrix Inequality (LMI) problem using Sum-of-Squares (SoS) programming~\citep{singh2019robust} or RKHS theory~\cite{singh2019learning} even when the model dynamics is unknown. However, such LMI-based techniques have to rely on an assumption on the special control input structure of a class of underactuated systems (see Sec.~\ref{sec:controller_design} for details) and therefore cannot be applied to general robotic systems. Moreover, the above methods only learn the CCM. The controller needs to be found separately by computing geodesics of the CCM, which cannot be solved exactly and has high computational complexity.
To use SoS, the system dynamics need to be polynomial equations or can be approximated as polynomial equations.
The degree and template of the polynomials in SoS play a crucial role in determining whether a solution exists and need to be tuned for each system. 
In~\citep{tsukamoto2020neural}, the authors proposed a synthesis framework using recurrent NNs to model the contraction metric and this framework works for nonlinear systems which are a convex combination of multiple state-dependent coefficients (i.e. $f(x,t)$ is written as $A(x,t)x$). Again, the controller in~\citep{tsukamoto2020neural} is constructed from the learned metric. 
In contrast, our approach can simultaneously synthesize the controller and the CCM certificate for control-affine systems without any additional structural assumptions.

We provide two methods to prove convergence guarantees of the learned controller and CCM (Section~\ref{sec:guarantees}). The first method provides deterministic guarantees by leveraging the Lipschitz continuity of the CCM’s condition.
Furthermore, we observed that even if the CCM’s condition does not hold globally for every state in the state space, the resulting trajectories still often converge to the reference trajectories in our experience. This motivates our second approach based on conformal regression
to give probabilistic guarantees on the convergence of the tracking error. Both methods can provide upper bounds on the tracking error, using which one can explore trajectory planning methods such as sampling-based methods (e.g. RRT~\citep{lavalle1998rapidly}, PRM~\citep{kavraki1998analysis}), model predictive control (MPC)~\citep{mpcPytorch}, and satisfiability-based methods~\citep{fan2020fast} to find safe reference trajectories to accomplish the safe motion planning missions. We compare our approach Certified Control using Contraction Metric (C3M) with the SoS programming~\citep{singh2019robust}, RL~\citep{schulman2017proximal}, and MPC~\citep{mpcPytorch,camacho2013model,anderson2010optimal} on several representative robotic systems, including the ones whose dynamics are learned from physical data using neural networks. We show that \toolname outperforms other methods by being the only approach that can find converging tracking controllers for all the benchmarks. We provide two metrics for evaluating the tracking performance and show that controllers achieved through \toolname have a much smaller tracking error. In addition,  controllers found by \toolname can be executed in sub-milliseconds at each step and is much faster than 
methods that only learn the metric (e.g. SoS programming) and online control methods (e.g.  MPC).

\section{Preliminaries and notations}
\label{sec:prelim}
We denote by $\reals$ and $\rp$ the set of real and non-negative real numbers respectively. For a symmetric matrix $A \in \reals^{n\times n}$, the notation $A \succ 0$ ($A \prec 0$) means $A$ is positive (negative) definite. The set of positive definite $n \times n$ matrices is denoted by $S_n^{>0}$. For a matrix-valued function $M(x):\reals^n \mapsto \reals^{n \times n}$, its element-wise Lie derivative along 
a vector $v \in \reals^{n}$ is $\partial_v M := \sum_{i} v^{i} \frac{\partial M}{\partial x^i}$. Unless otherwise stated, $x^i$ denotes the $i$-th element of vector $x$. For $A \in \reals^{n \times n}$, we denote $A + A^\T$ by $\hat{A}$.

\paragraph{Dynamical system}
We consider {\em control-affine} systems of the form
\begin{equation}
    \dot{x}(t) = f(x(t)) + B(x(t)) u(t) + d(t),
    \label{eq:sys}
\end{equation}
where $x(t) \in \X \subseteq \reals^n$, $u(t) \in \mathcal{U} \subseteq \reals^m$, and $d(t) \in \mathcal{D} \subseteq \reals^n$ for all $t \in \rp$ are states, inputs, and disturbances respectively. Here, $\X, \mathcal{U}$ and $\mathcal{D}$ are compact sets that represent the state, input, and disturbance space respectively. We assume that $f: \reals^n \mapsto \reals^n$ and $B: \reals^n \mapsto \reals^{n \times m}$ are smooth functions,  the control input $u : \rp \mapsto \mathcal{U}$ is a piece-wise continuous function, and the right-side of Eq.~\eqref{eq:sys} holds at discontinuities. Furthermore, we assume that the disturbance is bounded, i.e. $\sup_t \|d(t)\|_2 < \infty$. Given an input signal $: \rp \mapsto \mathcal{U}$, a disturbance signal $d:\rp \mapsto \mathcal{D}$, and an initial state $x_0 \in \X$, a \textit{(state) trajectory} of the system is a function $x: \rp\mapsto \reals^{n}$ such that $x(t),u(t),d(t)$ satisfy Eq.~\eqref{eq:sys} and $x(0) = x_0$. The goal of this paper is to design a tracking controller $u(\cdot)$ such that the controlled trajectory $x(t)$ can track \textit{any} target trajectory $x^*(t)$ generated by a reference control $u^*(t)$ when $x(0)$ is in a neighborhood of $x^{*}(0)$ and there is no disturbance (i.e. $d(t) = 0$). Furthermore, when $d(t) \neq 0$, the tracking error can also be bounded (Sec.~\ref{sec:robustness}). Fig.~\ref{fig:arch}~(a) illustrates the reference trajectory and the actual trajectory controlled by the proposed method.

\paragraph{Control contraction metric theory}
Contraction theory~\citep{lohmiller1998contraction} analyzes the incremental stability of a system by considering the evolution of the distance between any pair of arbitrarily close neighboring trajectories. Let us first consider a time-invariant autonomous system of the form $\dot{x} = f(x)$. Given a pair of neighboring trajectories, denote the infinitesimal displacement between them by $\dx$, which is also called a \textit{virtual displacement}. The evolution of $\dx$ is dominated by a linear time varying (LTV) system: $\ddx = \frac{\partial f}{\partial x}(x) \dx.$ Thus, the dynamics of the squared distance $\dx ^\T \dx$ is given by $\frac{d}{dt} (\dx ^ \T \dx) = 2 \dx^\T \ddx = 2 \dx^\T \frac{\partial f}{\partial x} \dx.$
If the symmetrical part of the Jacobian $\frac{\partial f}{\partial x}$ is uniformly negative definite, i.e. there exists a constant $\lambda > 0$ such that for all $x$, $\frac{1}{2}\reallywidehat{\frac{\partial f}{\partial x}} \preceq -\lambda \mathbf{I}$, then $\dx^\T \dx$ converges to zero exponentially at rate $2 \lambda$. Hence, all trajectories of this system will converge to a %
{common}
trajectory \citep{lohmiller1998contraction}. Such a system is referred to be \textit{contracting}.

The above analysis can be generalized by introducing a \textit{contraction metric} $M : \reals^n \mapsto S_n^{>0}$, which is a smooth function. Then, $\dx^\T M(x) \dx$ can be interpreted as a Riemannian squared length. Since $M(x) \succ 0$ for all $x$, if $\dx^\T M(x) \dx$ converges to $0$ exponentially, then the system is contracting. The converse is also true. 
As shown in 
\citep{lohmiller1998contraction}, if a system is contracting, then there exists a contraction metric $M(x)$ and a constant $\lambda > 0$ such that $\frac{d}{dt}(\dx^\T M(x) \dx) < - \lambda \dx^\T M(x) \dx$ for all $x$ and $\dx$.

Contraction theory can be further extended to control-affine systems. First, let us ignore the disturbances in system~\eqref{eq:sys}, then the dynamics of the corresponding virtual system is given by $\ddx = A(x,u) \dx + B(x) \delta_u$, where $A(x,u) := \frac{\partial f}{\partial x} + \sum_{i=1}^{m} u^i \frac{\partial b_i}{\partial x}$,  $b_i$ is the $i$-th column of $B$,  $u^i$ is the $i$-th element of $u$, and $\delta_u$ is the infinitesimal difference of $u$. A fundamental theorem in Control Contraction Metric (CCM) theory~\cite{manchester2017control} says if there exists a metric $M(x)$ such that the following conditions hold for all $x$ and some $\lambda > 0$,
\begin{align}
B_{\bot}^\T \left( - \partial_f W(x) + \reallywidehat{\frac{\partial f(x)}{\partial x} W(x)} + 2 \lambda W(x) \right)B_{\bot} \prec 0, \label{eq:c1} \\
B_{\bot}^\T \left( \partial_{b_j} W(x) - \reallywidehat{\frac{\partial b_j(x)}{\partial x} W(x)} \right) B_{\bot} = 0,\ j=1,\dots,m,
\label{eq:c2}
\end{align}

where $B_{\bot}(x)$ is an annihilator matrix of $B(x)$ satisfying $B_{\bot}^\T B = 0$, and $W(x) = M(x)^{-1}$ is the dual metric, then there exists a tracking controller $k(x, x^*, u^*)$ such that the closed-loop system controlled by $u = k(x, x^*, u^*)$ is contracting with rate $\lambda$ under metric $M(x)$~\citep{manchester2017control}, which means $x(t)$ converges to $x^*(t)$ exponentially. However, as shown in~\citep{manchester2017control}, either finding a valid metric or constructing a controller from a given metric is not straightforward. Hence, we proposed to jointly learn the metric and controller using neural networks.

\section{Learning tracking controllers using CCM with formal guarantees}
\begin{figure}
    \centering
\includegraphics[width=\textwidth,trim=70 70 110 60,clip]{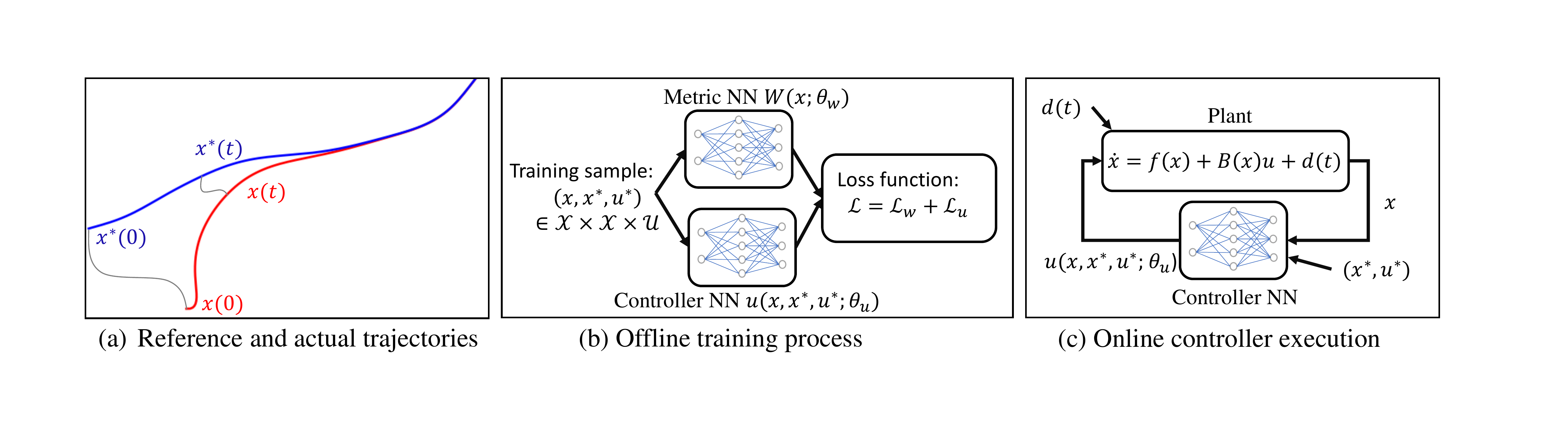}
    \caption{\footnotesize Components of the proposed method. (a) Trajectories of the Dubins car model, which is a classical nonholonomic system.
Here, $x^*$ is the reference trajectory, and $x$ is the actual trajectory controlled by our learned controller. (b)Joint training for the CCM and the controller. (c) Illustration on the online execution.}
    \label{fig:arch}
\end{figure}
In this paper, we utilize machine learning methods to jointly learn a Contraction Metric and a tracking controller for System~\eqref{eq:sys} with {\em known} dynamics.
The overall learning system is illustrated in Fig.~\ref{fig:arch} (b).
Different from RL-based methods, since all the conditions used for learning (Eq.~\eqref{eq:c1},\eqref{eq:c2}, and \eqref{eq:c3}) are defined for a single state at a specific time instant instead of the whole trajectory, running of the system is not required for training. Hence, the training is fully offline and fast.
The data set for training is of the form $\{(x_i, x^*_i, u^*_i) \in \X \times \X \times \U \}_{i=1}^{N}$, and samples are drawn independently. 
The contraction metric and the tracking controller are parameterized using neural networks. In what follows we first ignore the disturbance $d(t)$ and then show in Sec.~\ref{sec:robustness} that with $d(t)\neq 0$, the tracking error of the learned controller can still be bounded.

\subsection{Controller and metric learning}\label{sec:controller_design}
The controller $u(x, x^*, u^*; \theta_u)$ is a neural network with parameters $\theta_u$, and the dual metric $W(x;\theta_w)$ is also a neural network with parameters $\theta_w$. By design, the controller satisfies that $x = x^* \Rightarrow u(x, x^*, u^*; \theta_u) = u^*$ for all $\theta_u$, and $W(x;\theta_w)$ is a symmetric matrix satisfying $W(x;\theta_w) \succeq \underline{w}\mathbf{I}$ for all $x$ and all $\theta_w$, where $\underline{w}$ is a hyper-parameter. More details about the design can be found in Sec.~\ref{sec:exps}.
Plugging $u(x, x^*, u^*; \theta_u)$ into System~\eqref{eq:sys} with $d(t) = 0$, the dynamics of the generalized squared length of the virtual displacement under metric $M(x;\theta_w) = W(x;\theta_w)^{-1}$ is given by $\frac{d}{dt}(\dx^\T M(x;\theta_w) \dx) = \dx^\T(\dot{M} + \reallywidehat{M (A+BK)}) \dx$, where $K = \frac{\partial u}{\partial x}$, and $\dot{M} = \sum_{i=1}^{n} \frac{\partial M}{\partial x^i} \dot{x}^i$, or using the Lie derivative notation $\dot{M} = \partial_{f(x)+B(x)u} M$. A sufficient condition from~\cite{manchester2017control} for the closed-loop system being contracting is that there exists $\lambda>0$ such that for all $(x, x^*, u^*) \in \X\times\X\times\U$,
\begin{equation}
\dot{M} + \reallywidehat{M (A+BK)} + 2 \lambda M \prec 0.
\label{eq:c3}
\end{equation}
Since $x = x^* \Rightarrow u(x, x^*, u^*; \theta_u) = u^*$, the reference $(x^*(t), u^*(t))$ is a valid trajectory of the closed-loop system. If the system is furthermore contracting, then starting from any initial state, the trajectory converges to a common trajectory which is indeed $x^*(t)$ as formally stated below.
\begin{proposition}
\label{prop:contracting}
If condition~\eqref{eq:c3} holds and $\overline{m} \geq \underline{m} > 0$ satisfying $\underline{m}\mathbf{I} \preceq M(x) \preceq \overline{m}\mathbf{I}$ for all $x$, then the initial difference between the reference and the actual trajectory is exponentially forgotten, i.e.
$
\|x(t) - x^*(t)\|_2 \leq \sqrt{{\overline{m}}/{\underline{m}}}~ e^{- \lambda t}\|x(0)-x^*(0)\|_2.
$
\end{proposition}

Denoting the LHS of Eq.\eqref{eq:c3} by $C_u(x, x^*, u^*; \theta_w, \theta_u)$ and the uniform distribution over the data space $\mathcal{S} := \X\times\X\times\U$ by $\rho(\mathcal{S})$, then the \textit{contraction risk} of the system is defined as
\begin{equation}
\mathcal{L}_u(\theta_w, \theta_u) = \mathop{\mathbb{E}}\limits_{(x,x^*,u^*) \sim \rho(\mathcal{S})}{L_{PD}(-C_u(x, x^*, u^*;\theta_w, \theta_u))},
\label{eq:loss_controller}
\end{equation}
where $L_{PD}(\cdot) \geq 0$ is for penalizing non-positive definiteness, and $L_{PD}(A)=0$ iff. $A \succeq 0$. Obviously, $\mathcal{L}_u(\theta_w^*, \theta_u^*) = 0$ implies that $u(x,x^*,u^*;\theta_u^*)$ and $W(x;\theta_w^*)$ satisfy inequality~\eqref{eq:c3} exactly.

As shown in Proposition~\ref{prop:contracting}, the overshoot of the tracking error is affected by the condition number $\overline{m}/\underline{m}$ of the metric. Since the smallest eigenvalue of the dual metric is lower bounded by $\underline{w}$ by design, penalizing large condition numbers is equivalent to penalizing the largest eigenvalues, and thus the following risk function is used
\begin{equation}
\mathcal{L}_c(\theta_w) = \mathop{\mathbb{E}}\limits_{(x,x^*,u^*) \sim \rho(\mathcal{S})}{L_{PD}(\overline{w}\mathbf{I} - W(x;\theta_w))},
\label{eq:loss_condition_number}
\end{equation}
where $\overline{w}$ is a hyper-parameter.

However, jointly learning a metric and a controller by minimizing $\mathcal{L}_u$ solely is hard and leads to poor results. Inspired by the CCM theory, we add two auxiliary loss terms for the dual metric to mitigate the difficulty of minimizing $\mathcal{L}_u$. As shown in Sec.~\ref{sec:prelim}, conditions~\eqref{eq:c1} and~\eqref{eq:c2} are sufficient for a metric to be a valid CCM. Intuitively, imposing these constraints on the dual metric $W(x;\theta_w)$ provides more guidance for optimization. More discussion can be found in Appendix. Denoting the LHS of Eq.~\eqref{eq:c1} and~\eqref{eq:c2} by $C_1(x; \theta_w)$ and $\{C_2^j(x; \theta_w)\}_{j=1}^{m}$, the following risk functions are used
\begin{equation}\label{eq:loss_metric}
\mathcal{L}_{w1}(\theta_w) = \mathop{\mathbb{E}}\limits_{(x,x^*,u^*) \sim \rho(\mathcal{S})} L_{PD}(-C_1(x;\theta_w));
\mathcal{L}_{w2}(\theta_w) = \sum_{j=1}^{m} \mathop{\mathbb{E}}\limits_{(x,x^*,u^*) \sim \rho(\mathcal{S})} \|C_2^i(x;\theta_w)\|_F,
\end{equation}
where $\|\cdot\|_F$ is the Frobenius norm.

Please note that in some cases, condition \eqref{eq:c2} can be automatically satisfied for all $\theta_w$ by designing $W(\cdot)$ appropriately. As stated in~\citep{singh2019learning}, if the matrix $B(x)$ is sparse and of the form
    $B(x) = \begin{bmatrix}0_{(n-m) \times m}\\ b(x)\end{bmatrix}$,
where $b(x)$ is an invertible matrix, condition~\eqref{eq:c2} can be automatically satisfied for all $x$ and all $\theta_w$ by 
making the upper-left $(n-m) \times (n-m)$ block of $W(x; \theta_w)$ not a function of the last $(n-m)$ elements of $x$.
In~\cite{singh2019learning,singh2019robust}, such sparsity assumptions are necessary for their approach to work.
For dynamical models satisfying this assumption, we make use of this property by designing $W$ to satisfy such a structure and eliminate $\mathcal{L}_{w2}$ since it will always be $0$. For models not satisfying this assumption, we will show in Sec.~\ref{sec:exps} the impact of the loss term $\mathcal{L}_{w2}$.

In order to train the neural network using sampled data, the following \textit{empirical risk} function is used
\begin{multline}
\mathcal{L}(\theta_u,\theta_w) = \frac{1}{N} \sum_{i=1}^{N} \Big[ L_{PD}(-C_u(x_i, x^*_i, u^*_i;\theta_w, \theta_u)) + L_{PD}(-C_1(x_i;\theta_w)) \\+ \sum_{j=1}^{m} \|C_2^j(x_i;\theta_w))\|_F + L_{PD}(\overline{w}\mathbf{I} - W(x_i;\theta_w)) \Big],
\end{multline}
where $\{(x_i, x^*_i, u^*_i)\}_{i=1}^{N}$ are drawn independently from $\rho(\mathcal{S})$, and $L_{PD}$ is implemented as follows. Given a matrix $A \in \reals^{n \times n}$, we randomly sample $K$ points $\{p_i \in \reals^n~|~||p_i||_2 = 1\}_{i=1}^{K}$. Then, the loss function is calculated as $L_{PD}(A) = \frac{1}{K} \sum_{i=1}^{K} \min\{0, -p_i ^\T A p_i\}$.

\subsection{Theoretical convergence guarantees}\label{sec:guarantees}
As discussed in Sec.~\ref{sec:controller_design}, satisfaction of inequality \eqref{eq:c3} for all $(x,x^*,u^*) \in \mathcal{S}$ is a strong guarantee to show the validity of $M(x)$ as a contraction metric and therefore provide convergence guarantees. SMT solvers~\cite{chang2019neural} are valid tools for verifying the satisfaction of \eqref{eq:c3} on the uncountable set $\mathcal{S}$. However, SMT solvers for NN have poor scalability and cannot handle the neural networks used in our experiments. Moreover, we observed in experiments that even if inequality~\eqref{eq:c3} does not hold for all $(x,x^*,u^*) \in \mathcal{S}$, the learned controller still has perfect performance. Therefore, we introduce two approaches to provide guarantees on the satisfaction of inequality \eqref{eq:c3} or on the tracking error directly, which in practice are easier to validate and still give the needed level of assurance. 

\paragraph{Deterministic guarantees using Lipschitz continuity.}\label{sec:DG}
For a Lipschitz continuous function $f: \X \mapsto \reals$ with Lipschitz constant $L_f$, discretizing the domain $\X$ such that the distance between any grid point and its nearest neighbor is less than $\tau$, if $f(x_i) < -L_f \tau$ holds for all grid point $x_i$, then $f(x) < 0$ holds for all $x \in \X$. We show in the following proposition that the largest eigenvalue of the LHS of inequality~\eqref{eq:c3} indeed has a Lipschitz constant if both the dynamics and the learned controller and metric have Lipschitz constants.

\begin{proposition}\label{thm:lconstant}
Let $A$, $B$, $K$, and $M$ be functions of $x$, $x^*$, and $u^*$.
If $\dot{M}$, $ M$, $A$, $B$, and $K$ all have Lipschitz constants $L_{\dot{M}}, L_M$, $L_A$, $L_B$, and $L_K$ respectively, and $2$-norms of the last four are all bounded by $S_M$, $S_A$, $S_B$, and $S_K$ respectively, then the largest eigenvalue function
$\lambda_{\max}\left( \dot{M} + \reallywidehat{M(A+BK)} + 2 \lambda M \right)$ has a Lipschitz constant $L_{\dot{M}} + 2 \left(S_ML_A + S_A L_M + S_MS_BL_K +  S_BS_KL_M + S_M S_KL_B + \lambda L_M \right)$.

\end{proposition}
The proof of Proposition~\ref{thm:lconstant} is provided in Appendix. Combining discrete samples and the Lipschitz constant from Proposition~\ref{thm:lconstant} can guarantee the strict satisfaction of inequality~\eqref{eq:c3} on the uncountable set $\mathcal{S}$.
The  Lipschitz constants for the metric and controller NNs can be computed using various existing tools such as the methods in~\citep{fazlyab2019efficient}.
An example demonstrating the verification process is given in Appendix. Since the estimate of the Lipschitz constant is too conservative, it usually requires a huge number of samples to verify the learned controller. Fortunately, this process only requires forward computation of the NN and thus can be done relatively fast.

\paragraph{Probabilistic guarantees.}\label{sec:PG}
We observed that even if inequality~\eqref{eq:c3} does not hold for all points in $\mathcal{S}$, the learned controller can still drive all simulated trajectories to converge to reference trajectories. 
This motivates us to derive probabilistic guarantees on the convergence as an alternative in the trajectory space. The probabilistic guarantee is derived based on a simple result from conformal prediction~\citep{vovk2005algorithmic}.
Let us consider the process of evaluating a tracking controller using a \textit{quality metric function}. Given an \textit{evaluation configuration} including the initial state $x(0)$, and the initial reference state $x^{*}(0)$ and reference control signal $u^{*}(t)$, we can get the actual trajectory $x(t)$ controlled by the tracking controller. The quality metric function summarizes the tracking error curve $\|x^{*}(t) - x(t)\|_2$ into a scalar which can be viewed as a score for the tracking. Now, it is clear that the quality metric function is a mapping from the evaluation configuration to a score. The following proposition gives a probabilistic guarantee on the distribution of the quality score based on empirical observations.

\begin{proposition}\label{prop:quantile}
Given a set of $n$ i.i.d. evaluation configurations, let $\{m_i\}_{i=1}^{n}$ be the quality scores for each configuration. Then, for a new i.i.d. evaluation configuration and the corresponding quality score $m_{n+1}$, we have
$\Pr(m_{n+1} \geq q_{1-\alpha}) \leq 1 - \alpha, $
where $q_{1-\alpha}$ is the $(1-\alpha)$-th quantile of $\{m_i\}_{i=1}^{n}$.
\end{proposition}

Note that Proposition~\ref{prop:quantile} asserts guarantees on the marginal distribution of $m_{n+1}$, which should be distinguished from the conditional distribution. Such probabilistic guarantees work on the trajectory space and therefore can be applied to any method that finds tracking controllers.
In Sec.~\ref{sec:exps}, we evaluate all the methods in comparison by computing the probabilistic guarantees using either the average tracking error or the convergence rate as the quality metric.

\subsection{Robustness of the learned controller}
\label{sec:robustness}
When the closed-loop system run with disturbances, the tracking error is still bounded as follows. 
\begin{theorem}\label{thm:robustness}
Given $M(x)$ and $u(x,x^*,u^*)$ satisfying inequality~\eqref{eq:c3}, since $M(x) \in S_n^{> 0}$, there exist $\overline{m} \geq \underline{m} > 0$ such that $\underline{m}\mathbf{I} \preceq M(x) \preceq \overline{m}\mathbf{I}$ for all $x$. Assume that the disturbance is uniformly bounded as $\|d(t)\|_2 \leq \epsilon$. Now, for the same reference, considering the trajectories $x_1(t)$ and $x_2(t)$ of the unperturbed and the perturbed closed-loop system respectively, the distance between these two trajectories can be bounded as $\|x_1(t) - x_2(t)\|_2 \leq \frac{R_0}{\sqrt{\underline{m}}} e^{-\lambda t} + \sqrt{\frac{\overline{m}}{\underline{m}}} \cdot \frac{\epsilon}{\lambda} (1 - e^{- \lambda t})$,
where $R_0 = \int_{x_1(0)}^{x_2(0)} \sqrt{\dx^\T M(x) \dx}$ is the geodesic distance between $x_1(0)$ and $x_2(0)$ under metric $M(x)$.
\end{theorem}
The proof of Theorem~\ref{thm:robustness} is inspired by Theorem 1 in~\cite{tsukamoto2020neural} and is provided in Appendix. In practice $\overline{m}$ and $\underline{m}$ can be computed using samples and the Lipschitz constant of the eigenvalues of $M(x)$ as discussed in Sec.~\ref{sec:guarantees}.

\section{Evaluation of performance}\label{sec:exps}

We evaluate the C3M approach on 9 representative case studies (5 of them are reported in Appendix), including high-dimensional (up to $9$ state variables and $3$ control variables) models and a system with learned dynamics represented by a neural network. All our experiments were conducted on a Linux workstation with two Xeon Silver 4110 CPUs, 32 GB RAM, and a Quadro P5000 GPU. Please note that the execution time reported in Tab.~\ref{tab:results} were all evaluated on CPU.

\textbf{Comparison methods.} We compare the performance of \toolname with $3$ different leading approaches for synthesizing tracking controllers, including both model-based and model-free methods. To be specific, the methods include
\begin{inparaenum}[(1.)]
    \item \textbf{SoS:} the SoS-based method proposed in \citep{singh2019robust}, using the official implementation~\footnote{\url{https://github.com/StanfordASL/RobustMP}}. 
    \item \textbf{MPC:} 
    an open-source implementation of MPC on PyTorch~\citep{mpcPytorch}, which solves nonlinear implicit MPC problems using the first-order approximation of the dynamics~\citep{tassa2014control}.
    \item \textbf{RL:} 
    the Proximal Policy Optimization (PPO) RL algorithm~\citep{schulman2017proximal} implemented in Tianshou~\footnote{\url{https://github.com/thu-ml/tianshou}}. 
\end{inparaenum}
Note that although RL is often used in a model-free setting, the advances in deep neural networks have made RL an outstanding tool for controller learning.

\textbf{Studied systems.} We study 4 representative system models adopted from classical benchmarks:
\begin{inparaenum}[(1).]
    \item \textbf{PVTOL} models a planar vertical-takeoff-vertical-landing (PVTOL) system for drones and is adopted from~\citep{singh2019robust, singh2019learning}. 
    \item \textbf{Quadrotor} models a physical quadrotor and is adopted from~\citep{singh2019robust}. 
    \item \textbf{Neural lander} models a drone flying close to the ground so that ground effect is prominent and therefore could not be ignored~\citep{liu2019robust}. 
    Note that in this model the ground effect is learned from empirical data of a physical drone using a $4$-layer neural network~\citep{liu2019robust}. Due to the neural network function in the dynamics, the SoS-based method as in~\citep{singh2019robust} cannot handle such a model since neural networks are hard to be approximated by polynomials. %
    \item \textbf{SEGWAY} models a real-world Segway robot and is adopted from \citep{taylor2019learning}. 
    Since this model does not satisfy the sparsity assumption for matrix $B$ in Sec.~\ref{sec:controller_design}, SoS-based method cannot handle it. 
    With this model, we also study the impact of the regularization term $\mathcal{L}_{w2}$ in Equation~\eqref{eq:loss_metric} for learning a tracking controller.
\end{inparaenum}

We report the performance of \toolname on $5$ other examples in Appendix, including a 10D quadrotor model, a cart-pole model, a pendulum model, a two-link planar robot arm system, and the Dubin's vehicle model in Fig.~\ref{fig:arch}. The detailed dynamics of the above 4 representative system models are also reported in Appendix.

\textbf{Implementation details.} For all systems studied here, we model the dual metric as $W(x) = C(x; \theta_w)^\T C(x; \theta_w) + \underline{w} \mathbf{I}$ where $C(x; \theta_w) \in \reals^{n \times n}$ is a 2-layer neural network, of which the hidden layer contains 128 neurons. We use a relative complex structure for the controller. First, two weight matrices of the controller $w_1 = w_1(x, x^*; \theta_{u1})$ and $w_2 = w_2(x, x^*; \theta_{u2})$ are modeled using two 2-layer neural networks with 128 neurons in the hidden layer, where $\theta_{u1}$ and $\theta_{u2}$ are the parameters. Then the controller is given by $u(x,x^*,u^*; \theta_u) = w_2  \cdot \tanh(w_1 \cdot (x-x^*)) + u^*$, where $\theta_u = \{\theta_{u1}, \theta_{u2}\}$ and $\tanh(\cdot)$ is the hyperbolic tangent function. It is easy to verify the assumptions made in Sec.~\ref{sec:controller_design} are satisfied: For the controller, $x = x^* \Rightarrow u(x, x^*, u^*; \theta_u) = u^*$ for all $\theta_u$. For the metric, $W(x;\theta_w)$ is a symmetric matrix satisfying $W(x;\theta_w) \succeq \underline{w}\mathbf{I}$ for all $x$ and all $\theta_w$. A training set with 130K samples is used. We train the NN for $20$ epochs with the Adam~\citep{kingma2014adam} optimizer.

The process of generating random reference trajectories $x^*$ is critical in the evaluation of tracking controllers and also a fundamental part in RL training.
We pre-define a group of sinusoidal signals with some fixed frequencies and randomly sampled a weight for each frequency component, then the reference control inputs $u^*(t)$ are calculated as the linear combination of the sinusoidal signals. The initial states of the reference trajectories $x^*(0)$ are uniformly randomly sampled from a compact set. Using $x^*(0)$ and $u^*(t)$, we can get the reference trajectories $x^{*}(t)$ following the system dynamics~\eqref{eq:sys} with $d(t)=0$. For evaluation, the initial errors $x_e(0)$ are uniformly randomly sampled from a bounded set, and the initial states $x(0)$ are computed as $x(0) = x^*(0) + x_e(0)$. Trajectories of the closed-loop systems are then simulated on a bounded time horizon $[0, T]$.

To make use of the model-free PPO RL library, we formulate some key concepts as follows. The state of the environment at time $t$ is the concatenation of $x(t)$ and $x^*(t)$. At the beginning of each episode, the environment randomly samples two initial points as $x(0)$ and $x^*(0)$ respectively, and a reference control input $u^*(t)$ using the aforementioned sampling method. At each transition, the environment takes the action $u(t)$ from the agent, and returns the next state as $x(t+\Delta t) = x(t) + \Delta t (f(x(t)) + B(x(t))(u(t)+u^*(t)))$ and $x^*(t+\Delta t) = x^*(t) + \Delta t (f(x^*(t)) + B(x^*(t))u^*(t))$, and the reward as $r_t = 1 / (1 + ||w \circ (x(t+\Delta t) - x^*(t+\Delta t))||_2)$, where $w$ are predefined weights for each state. For a fair comparison, RL and our method share the same architecture for the controller.

For the MPC-PyTorch library, we observe that the time horizon, also called the receding time window for computing the control sequence, plays a crucial role. With smaller time window, the resulting controller cannot always produce trajectories that converge to the reference trajectory. The larger the time window is, the smaller the tracking errors are. However, a large time window can also cause the computational time to increase. As a trade-off, we set time window to be $50$ time units in all cases.

\begin{table}[h!]
    \centering
    \caption{\footnotesize Comparison results of \toolname vs. other methods. $n$ and $m$ are dimensions of the state and input space.}
    \resizebox{\columnwidth}{!}{
    \begin{tabular}{c|c|c|c|c|c|c|c|c|c|c|c|c|c|c}
         \toprule
         \multirow{2}{*}{Model} & \multicolumn{2}{c|}{Dim} & \multicolumn{4}{c|}{Execution time per step (ms)} & \multicolumn{4}{c|}{Tracking error (AUC)} &  \multicolumn{4}{c}{Convergence rate ($\lambda/C$)}\\
         \cline{2-15}
         & $n$ & $m$ & \toolname & SoS & MPC & RL & \toolname & SoS & MPC & RL & \toolname & SoS & MPC & RL\\
         \hline
         PVTOL & 6 & 2 & 0.41 & 3.4 & 1968 & 0.41 & 0.659 & 0.975 & 0.892 & 0.735 & 1.153/3.95 & 0.864/4.91 & 0.489/4.95 & 0.799/3.84\\
         \hline
         Quadrotor & 9 & 3 & 0.40 & 12.6 & 3535 & 0.40 & 0.772 & 1.103 & 0.977 & 1.416 & 1.078/3.63 & 0.821/3.30 & 0.401/2.23 & 0.187/1.37 \\
         \hline
         Neural lander & 6 & 3 & 0.36 & - & 3385 & 0.36 & 0.588 & - & 0.713 & 0.793 & 1.724/2.89 & - & 0.822/1.34 & 0.606/1.30 \\
         \hline
         Segway & 4 & 1 & 0.29 & - & - & 0.29 & 0.704 & - & - & 1.408 & 0.446/3.11 & - & - & 0.168/8.12\\
         \bottomrule
    \end{tabular}
    }
    \label{tab:results}
\end{table}

\begin{figure}[t]
\centering
    \includegraphics[width=0.24\textwidth,trim=0 0 0 0,clip]{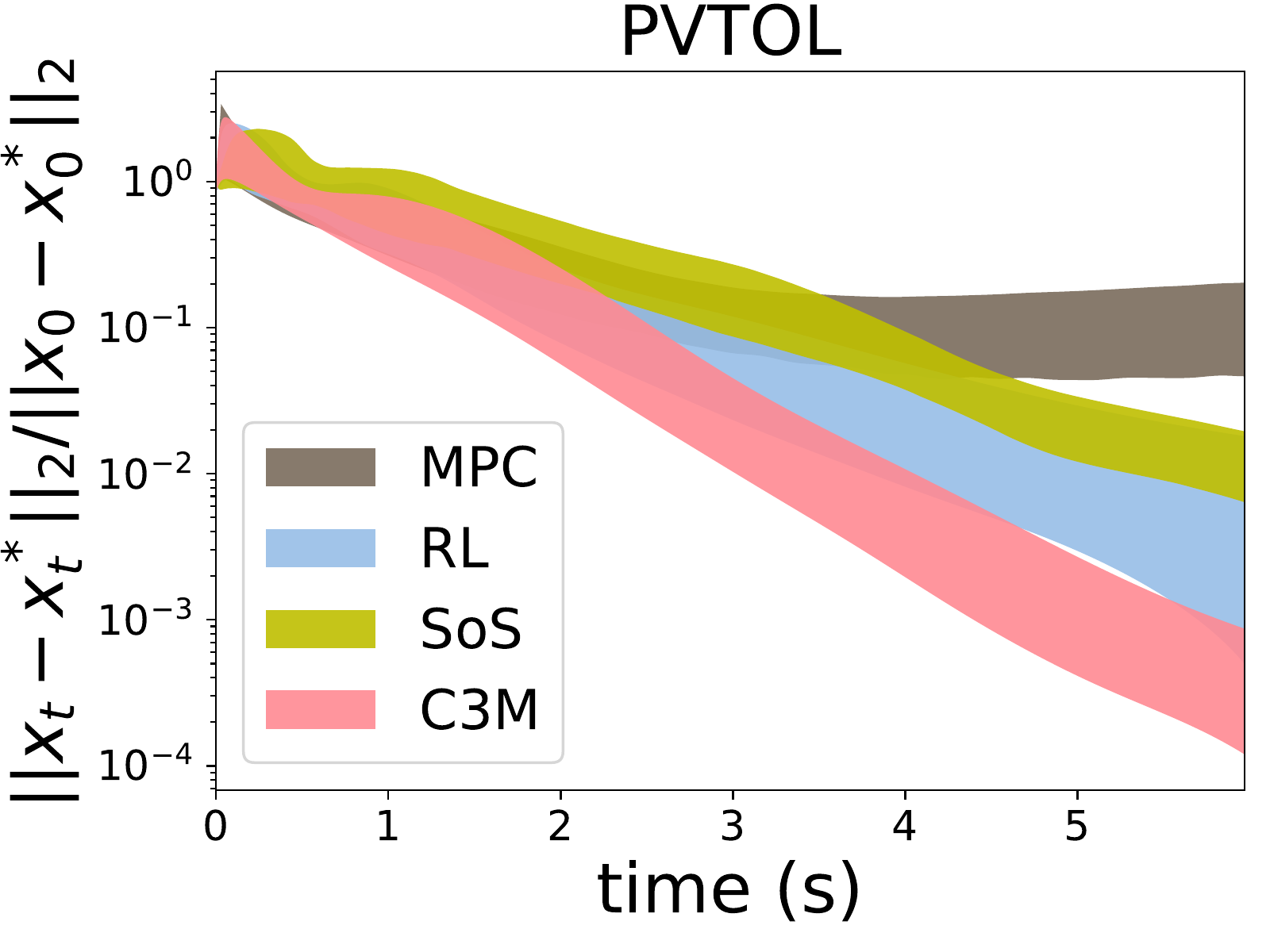}
    \includegraphics[width=0.24\textwidth,trim=0 0 0 0,clip]{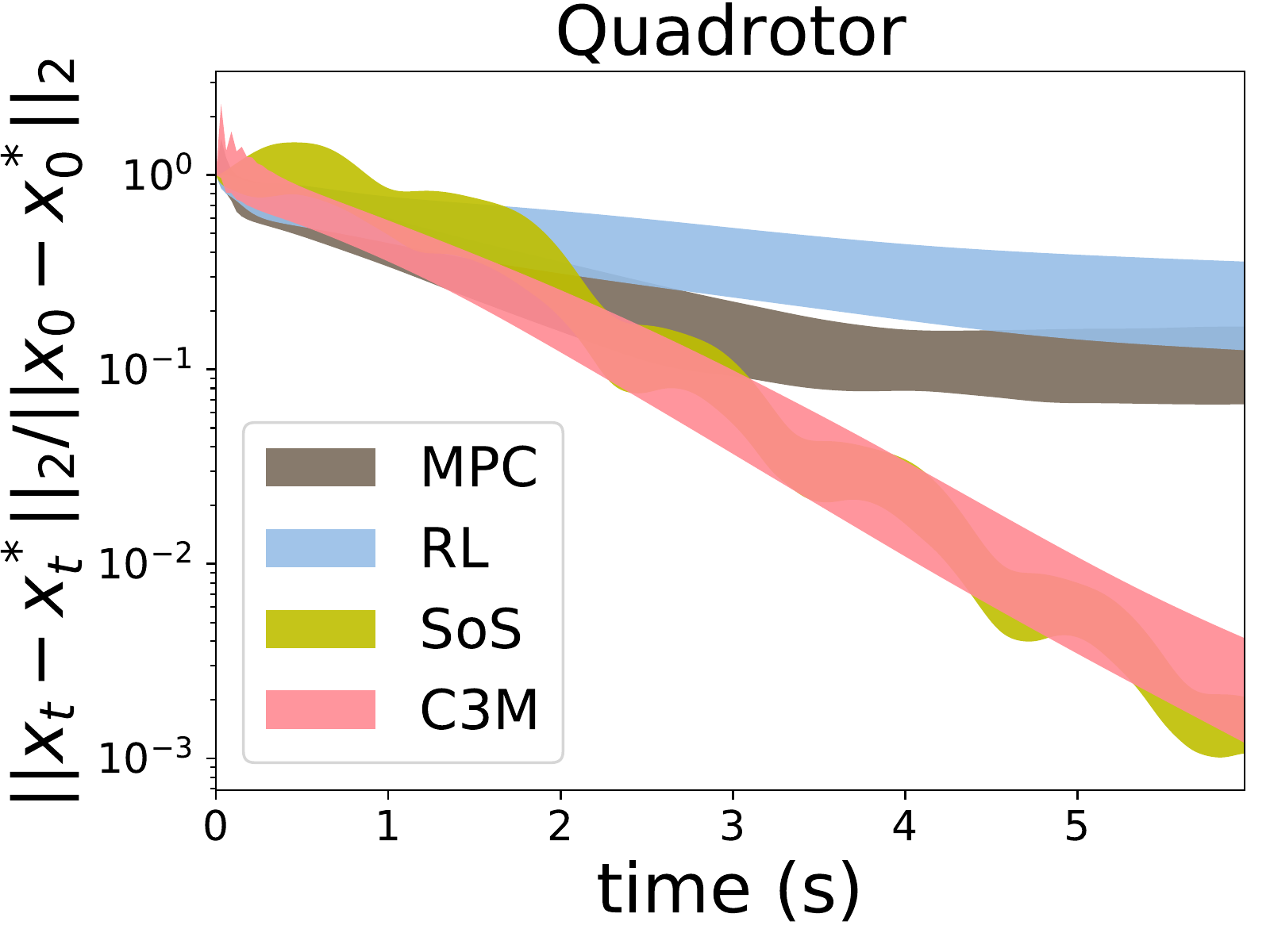}
    \includegraphics[width=0.24\textwidth,trim=0 0 0 0,clip]{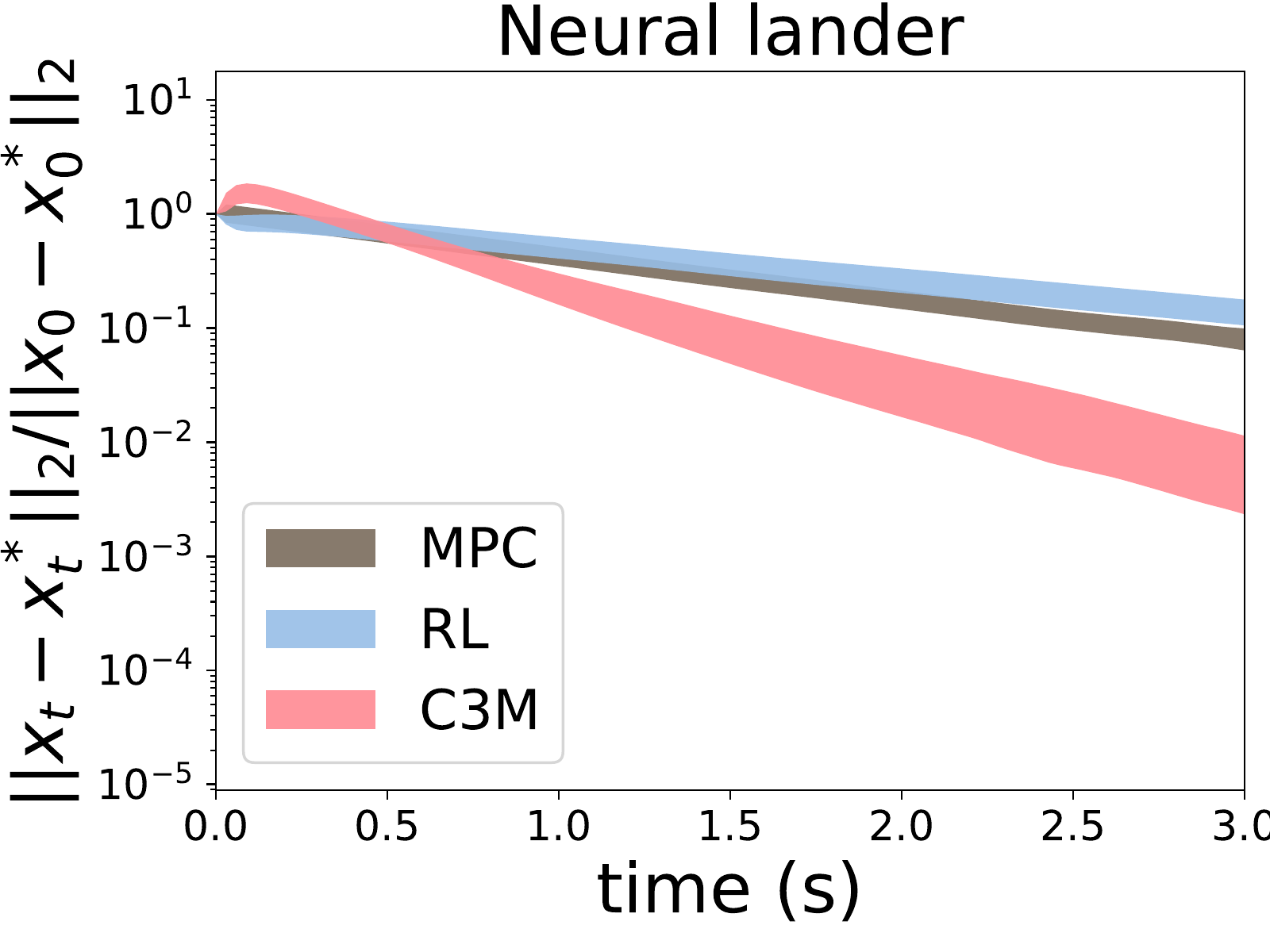}    \includegraphics[width=0.24\textwidth,trim=0 0 0 0,clip]{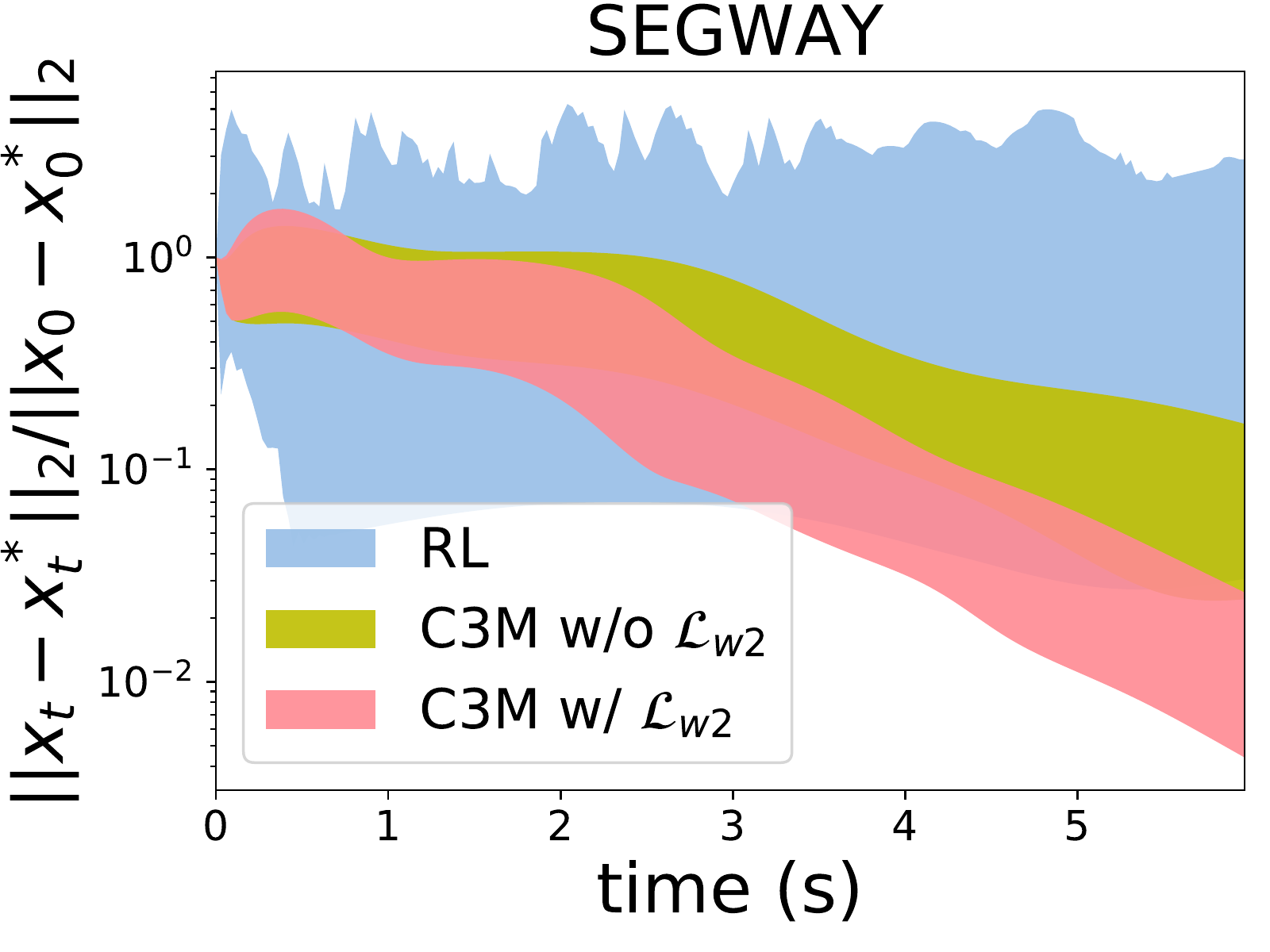}
\caption{\footnotesize Normalized tracking error on the benchmarks using different approaches. The $y$ axes are in log-scale. The tubes are tracking errors between mean plus and minus one standard deviation over 100 trajectories.}
\label{fig:results}
\end{figure}

\textbf{Results and discussion}  Results are reported in Tab.~\ref{tab:results} and Fig.~\ref{fig:results}. For each benchmark, a group of evaluation configurations was randomly sampled using the aforementioned sampling method, and then each method was evaluated with these configurations. 
In Fig.~\ref{fig:results} we show the normalized tracking error $x_e(t) = ||x(t) - x^*(t)||_2 / ||x(0) - x^*(0)||_2$ for each model and each available method. In Tab.~\ref{tab:results}, we evaluate the tracking controllers using two quality metrics as discussed in Sec.~\ref{sec:PG}. (1)~Total tracking error: Given the normalized tracking error curve $x_e(t)$ for $t \in [0, T]$, the total tracking error is just the area under this curve (AUC). (2)~Convergence rate: Given a set of tracking error curves, first we do the following for the error curve $x_e(t)$ that possesses the highest overshoot: We search for the convergence rate $\lambda > 0$ and the overshoot $C \geq 1$ such that $x_e(t) \leq Ce^{-\lambda t}$ for all $t \in [0, T]$ and the AUC of $Ce^{-\lambda t}$ is minimized. Then, $C$ is fixed, and the convergence rate $\lambda$ is calculated for each curve. After obtaining the quality scores for all the error curves, we report the $(1-\alpha)$-th quantile with $\alpha= 0.05$ as in Sec.~\ref{sec:PG}. %

Some observations are in order.
\begin{inparaenum}[(1)]
    \item On PVTOL and the quadrotor models, both SOS and our method successfully found tracking controllers such that the tracking error decreases rapidly.
    However, the SOS-based method entails the calculation of geodesics on the fly for control synthesis, while our method learns a neural controller offline, which makes a difference in running time. Also, due to the aforementioned limitations on control matrix sparsity, the SOS-based method cannot be applied to the neural lander and Segway. 
    \item The performance of the RL method varies with tasks. On PVTOL and the neural lander models, RL achieved comparable results. However, since the state space of the quadrotor model is larger, RL failed to find a contracting tracking controller within a reasonable time. As for the Segway model, although we have tried our best to tune the hyper-parameters, RL failed to find a reasonable controller.
\item Running time of MPC is not practical for online control tasks. Also, the MPC library used in experiments failed to handle the Segway model due to numerical issues.
\item The results of our methods with and without penalty term $\mathcal{L}_{w2}$ on Segway verify the sufficiency of this cost term when the sparsity assumption for $B$ is not satisfied.
\end{inparaenum}

\section{Conclusion}
In this paper, we showed that a certified tracking controller with guaranteed convergence and tracking error bounds can indeed be learned by learning a control contraction metric simultaneously as guidance. Results indicated that such a learned tracking controller can better track the reference or nominal trajectory, and the execution time of such a controller is much shorter than online control methods such as MPC. In future works, we plan to extend our learning framework to systems with unknown dynamics and control non-affine systems.

\acknowledgments{The authors acknowledge support from the 
DARPA Assured Autonomy under 
contract FA8750-19-C-0089,
U.S. Army Research Laboratory Cooperative Research Agreement  W911NF-17-2-0196, U.S. National Science
Foundation(NSF) grants \#1740079 and \#1750009. The views, opinions and/or findings expressed are those of the author(s) and should not be interpreted as representing the official views or policies of the Department of Defense or the U.S. Government.}

\appendix
\bibliography{references}  %

\begin{thebibliography}{66}
\providecommand{\natexlab}[1]{#1}
\providecommand{\url}[1]{\texttt{#1}}
\expandafter\ifx\csname urlstyle\endcsname\relax
  \providecommand{\doi}[1]{doi: #1}\else
  \providecommand{\doi}{doi: \begingroup \urlstyle{rm}\Url}\fi

\bibitem[LaValle(2006)]{lavalle2006planning}
S.~M. LaValle.
\newblock \emph{Planning algorithms}.
\newblock Cambridge university press, 2006.

\bibitem[Kavraki et~al.(1998)Kavraki, Kolountzakis, and
  Latombe]{kavraki1998analysis}
L.~E. Kavraki, M.~N. Kolountzakis, and J.-C. Latombe.
\newblock Analysis of probabilistic roadmaps for path planning.
\newblock \emph{IEEE Transactions on Robotics and Automation}, 14\penalty0
  (1):\penalty0 166--171, 1998.

\bibitem[Herbert et~al.(2017)Herbert, Chen, Han, Bansal, Fisac, and
  Tomlin]{herbert2017fastrack}
S.~L. Herbert, M.~Chen, S.~Han, S.~Bansal, J.~F. Fisac, and C.~J. Tomlin.
\newblock {FaSTrack}: {A} modular framework for fast and guaranteed safe motion
  planning.
\newblock In \emph{56th Annual Conference on Decision and Control (CDC)}. IEEE,
  2017.

\bibitem[Vaskov et~al.(2019)Vaskov, Kousik, Larson, Bu, Ward, Worrall,
  Johnson-Roberson, and Vasudevan]{vaskov2019towards}
S.~Vaskov, S.~Kousik, H.~Larson, F.~Bu, J.~Ward, S.~Worrall,
  M.~Johnson-Roberson, and R.~Vasudevan.
\newblock Towards provably not-at-fault control of autonomous robots in
  arbitrary dynamic environments.
\newblock \emph{arXiv preprint arXiv:1902.02851}, 2019.

\bibitem[Majumdar and Tedrake(2017)]{majumdar2017funnel}
A.~Majumdar and R.~Tedrake.
\newblock Funnel libraries for real-time robust feedback motion planning.
\newblock \emph{The International Journal of Robotics Research}, 36\penalty0
  (8):\penalty0 947--982, 2017.

\bibitem[Jha et~al.(2018)Jha, Raj, Jha, and Shankar]{jha2018duality}
S.~Jha, S.~Raj, S.~K. Jha, and N.~Shankar.
\newblock Duality-based nested controller synthesis from {STL} specifications
  for stochastic linear systems.
\newblock In \emph{International Conference on Formal Modeling and Analysis of
  Timed Systems}, pages 235--251. Springer, 2018.

\bibitem[Amos et~al.(2018)Amos, Jimenez, Sacks, Boots, and Kolter]{mpcPytorch}
B.~Amos, I.~Jimenez, J.~Sacks, B.~Boots, and J.~Z. Kolter.
\newblock Differentiable mpc for end-to-end planning and control.
\newblock In \emph{Advances in Neural Information Processing Systems}, 2018.

\bibitem[Singh et~al.(2019)Singh, Richards, Sindhwani, Slotine, and
  Pavone]{singh2019learning}
S.~Singh, S.~M. Richards, V.~Sindhwani, J.-J.~E. Slotine, and M.~Pavone.
\newblock Learning stabilizable nonlinear dynamics with contraction-based
  regularization.
\newblock \emph{arXiv:1907.13122}, 2019.

\bibitem[Chang et~al.(2019)Chang, Roohi, and Gao]{chang2019neural}
Y.-C. Chang, N.~Roohi, and S.~Gao.
\newblock Neural lyapunov control.
\newblock In \emph{Advances in Neural Information Processing Systems}, pages
  3245--3254, 2019.

\bibitem[Ravanbakhsh and Sankaranarayanan(2016)]{ravanbakhsh2016robust}
H.~Ravanbakhsh and S.~Sankaranarayanan.
\newblock Robust controller synthesis of switched systems using counterexample
  guided framework.
\newblock In \emph{2016 international conference on embedded software
  (EMSOFT)}, pages 1--10. IEEE, 2016.

\bibitem[Ames et~al.(2014)Ames, Grizzle, and Tabuada]{ames2014control}
A.~D. Ames, J.~W. Grizzle, and P.~Tabuada.
\newblock Control barrier function based quadratic programs with application to
  adaptive cruise control.
\newblock In \emph{53rd IEEE Conference on Decision and Control}, pages
  6271--6278. IEEE, 2014.

\bibitem[Ames et~al.(2019)Ames, Coogan, Egerstedt, Notomista, Sreenath, and
  Tabuada]{ames2019control}
A.~D. Ames, S.~Coogan, M.~Egerstedt, G.~Notomista, K.~Sreenath, and P.~Tabuada.
\newblock Control barrier functions: Theory and applications.
\newblock In \emph{2019 18th European Control Conference (ECC)}, pages
  3420--3431. IEEE, 2019.

\bibitem[Taylor et~al.(2019)Taylor, Singletary, Yue, and
  Ames]{taylor2019learning}
A.~Taylor, A.~Singletary, Y.~Yue, and A.~Ames.
\newblock Learning for safety-critical control with control barrier functions.
\newblock \emph{arXiv preprint arXiv:1912.10099}, 2019.

\bibitem[Tedrake(2009)]{tedrake2009lqr}
R.~Tedrake.
\newblock {LQR-Trees}: Feedback motion planning on sparse randomized trees.
\newblock 2009.

\bibitem[Bansal et~al.(2017)Bansal, Chen, Fisac, and Tomlin]{bansal2017safe}
S.~Bansal, M.~Chen, J.~F. Fisac, and C.~J. Tomlin.
\newblock Safe sequential path planning of multi-vehicle systems under presence
  of disturbances and imperfect information.
\newblock In \emph{American Control Conference}, 2017.

\bibitem[Kapinski et~al.(2014)Kapinski, Deshmukh, Sankaranarayanan, and
  Arechiga]{kapinski2014simulation}
J.~Kapinski, J.~V. Deshmukh, S.~Sankaranarayanan, and N.~Arechiga.
\newblock Simulation-guided lyapunov analysis for hybrid dynamical systems.
\newblock In \emph{Proceedings of the 17th international conference on Hybrid
  systems: computation and control}, pages 133--142, 2014.

\bibitem[Khansari-Zadeh and Billard(2014)]{khansari2014learning}
S.~M. Khansari-Zadeh and A.~Billard.
\newblock Learning control lyapunov function to ensure stability of dynamical
  system-based robot reaching motions.
\newblock \emph{Robotics and Autonomous Systems}, 2014.

\bibitem[Jha and Lincoln(2018)]{jha2018data}
S.~Jha and P.~Lincoln.
\newblock Data efficient learning of robust control policies.
\newblock In \emph{2018 56th Annual Allerton Conference on Communication,
  Control, and Computing (Allerton)}, pages 856--861. IEEE, 2018.

\bibitem[Jha et~al.(2017)Jha, Tiwari, Seshia, Sahai, and Shankar]{jha2017telex}
S.~Jha, A.~Tiwari, S.~A. Seshia, T.~Sahai, and N.~Shankar.
\newblock Telex: Passive stl learning using only positive examples.
\newblock In \emph{International Conference on Runtime Verification}, pages
  208--224. Springer, 2017.

\bibitem[Jha et~al.(2018)Jha, Raman, Sadigh, and Seshia]{jha2018safe}
S.~Jha, V.~Raman, D.~Sadigh, and S.~A. Seshia.
\newblock Safe autonomy under perception uncertainty using chance-constrained
  temporal logic.
\newblock \emph{Journal of Automated Reasoning}, 60\penalty0 (1):\penalty0
  43--62, 2018.

\bibitem[Taylor et~al.(2019)Taylor, Dorobantu, Le, Yue, and
  Ames]{taylor2019episodic}
A.~J. Taylor, V.~D. Dorobantu, H.~M. Le, Y.~Yue, and A.~D. Ames.
\newblock Episodic learning with control lyapunov functions for uncertain
  robotic systems.
\newblock \emph{arXiv:1903.01577}, 2019.

\bibitem[Robey et~al.(2020)Robey, Hu, Lindemann, Zhang, Dimarogonas, Tu, and
  Matni]{robey2020learning}
A.~Robey, H.~Hu, L.~Lindemann, H.~Zhang, D.~V. Dimarogonas, S.~Tu, and
  N.~Matni.
\newblock Learning control barrier functions from expert demonstrations.
\newblock \emph{arXiv preprint arXiv:2004.03315}, 2020.

\bibitem[Choi et~al.(2020)Choi, Casta{\~n}eda, Tomlin, and
  Sreenath]{choi2020reinforcement}
J.~Choi, F.~Casta{\~n}eda, C.~J. Tomlin, and K.~Sreenath.
\newblock Reinforcement learning for safety-critical control under model
  uncertainty, using control lyapunov functions and control barrier functions.
\newblock \emph{arXiv preprint arXiv:2004.07584}, 2020.

\bibitem[Polydoros and Nalpantidis(2017)]{polydoros2017survey}
A.~S. Polydoros and L.~Nalpantidis.
\newblock Survey of model-based reinforcement learning: Applications on
  robotics.
\newblock \emph{Journal of Intelligent \& Robotic Systems}, 86\penalty0
  (2):\penalty0 153--173, 2017.

\bibitem[Ohnishi et~al.(2019)Ohnishi, Wang, Notomista, and
  Egerstedt]{ohnishi2019barrier}
M.~Ohnishi, L.~Wang, G.~Notomista, and M.~Egerstedt.
\newblock Barrier-certified adaptive reinforcement learning with applications
  to brushbot navigation.
\newblock \emph{IEEE Transactions on robotics}, 2019.

\bibitem[Deisenroth and Rasmussen(2011)]{deisenroth2011pilco}
M.~Deisenroth and C.~E. Rasmussen.
\newblock Pilco: A model-based and data-efficient approach to policy search.
\newblock In \emph{Proceedings of the 28th International Conference on machine
  learning (ICML-11)}, pages 465--472, 2011.

\bibitem[Chua et~al.(2018)Chua, Calandra, McAllister, and Levine]{chua2018deep}
K.~Chua, R.~Calandra, R.~McAllister, and S.~Levine.
\newblock Deep reinforcement learning in a handful of trials using
  probabilistic dynamics models.
\newblock In \emph{Advances in Neural Information Processing Systems}, pages
  4754--4765, 2018.

\bibitem[Manchester and Slotine(2017)]{manchester2017control}
I.~R. Manchester and J.-J.~E. Slotine.
\newblock Control contraction metrics: Convex and intrinsic criteria for
  nonlinear feedback design.
\newblock \emph{IEEE Transactions on Automatic Control}, 2017.

\bibitem[Manchester et~al.(2015)Manchester, Tang, and
  Slotine]{manchester2015unifying}
I.~R. Manchester, J.~Z. Tang, and J.-J.~E. Slotine.
\newblock Unifying classical and optimization-based methods for robot tracking
  control with control contraction metrics.
\newblock In \emph{International Symposium on Robotics Research (ISRR)}, pages
  1--16, 2015.

\bibitem[Singh et~al.(2019)Singh, Landry, Majumdar, Slotine, and
  Pavone]{singh2019robust}
S.~Singh, B.~Landry, A.~Majumdar, J.-J. Slotine, and M.~Pavone.
\newblock Robust feedback motion planning via contraction theory.
\newblock \emph{The International Journal of Robotics Research}, 2019.

\bibitem[Tsukamoto and Chung(2020)]{tsukamoto2020neural}
H.~Tsukamoto and S.-J. Chung.
\newblock Neural contraction metrics for robust estimation and control: A
  convex optimization approach.
\newblock \emph{arXiv preprint arXiv:2006.04361}, 2020.

\bibitem[LaValle(1998)]{lavalle1998rapidly}
S.~M. LaValle.
\newblock Rapidly-exploring random trees: A new tool for path planning.
\newblock 1998.

\bibitem[Fan et~al.(2020)Fan, Miller, and Mitra]{fan2020fast}
C.~Fan, K.~Miller, and S.~Mitra.
\newblock Fast and guaranteed safe controller synthesis for nonlinear vehicle
  models.
\newblock In \emph{International Conference on Computer Aided Verification},
  2020.

\bibitem[Schulman et~al.(2017)Schulman, Wolski, Dhariwal, Radford, and
  Klimov]{schulman2017proximal}
J.~Schulman, F.~Wolski, P.~Dhariwal, A.~Radford, and O.~Klimov.
\newblock Proximal policy optimization algorithms.
\newblock \emph{arXiv preprint arXiv:1707.06347}, 2017.

\bibitem[Camacho and Alba(2013)]{camacho2013model}
E.~F. Camacho and C.~B. Alba.
\newblock \emph{Model predictive control}.
\newblock Springer, 2013.

\bibitem[Anderson et~al.(2010)Anderson, Peters, Pilutti, and
  Iagnemma]{anderson2010optimal}
S.~J. Anderson, S.~C. Peters, T.~E. Pilutti, and K.~Iagnemma.
\newblock An optimal-control-based framework for trajectory planning, threat
  assessment, and semi-autonomous control of passenger vehicles in hazard
  avoidance scenarios.
\newblock \emph{International Journal of Vehicle Autonomous Systems},
  8\penalty0 (2-4):\penalty0 190--216, 2010.

\bibitem[Lohmiller and Slotine(1998)]{lohmiller1998contraction}
W.~Lohmiller and J.-J.~E. Slotine.
\newblock On contraction analysis for non-linear systems.
\newblock \emph{Automatica}, 34\penalty0 (6):\penalty0 683--696, 1998.

\bibitem[Fazlyab et~al.(2019)Fazlyab, Robey, Hassani, Morari, and
  Pappas]{fazlyab2019efficient}
M.~Fazlyab, A.~Robey, H.~Hassani, M.~Morari, and G.~J. Pappas.
\newblock Efficient and accurate estimation of lipschitz constants for deep
  neural networks.
\newblock In \emph{Advances in Neural Information Processing Systems
  (NeurIPS)}, 2019.

\bibitem[Vovk et~al.(2005)Vovk, Gammerman, and Shafer]{vovk2005algorithmic}
V.~Vovk, A.~Gammerman, and G.~Shafer.
\newblock \emph{Algorithmic learning in a random world}.
\newblock Springer Science \& Business Media, 2005.

\bibitem[Tassa et~al.(2014)Tassa, Mansard, and Todorov]{tassa2014control}
Y.~Tassa, N.~Mansard, and E.~Todorov.
\newblock Control-limited differential dynamic programming.
\newblock In \emph{IEEE International Conference on Robotics and Automation},
  pages 1168--1175. IEEE, 2014.

\bibitem[Liu et~al.(2019)Liu, Shi, Chung, Anandkumar, and Yue]{liu2019robust}
A.~Liu, G.~Shi, S.-J. Chung, A.~Anandkumar, and Y.~Yue.
\newblock Robust regression for safe exploration in control.
\newblock \emph{arXiv preprint arXiv:1906.05819}, 2019.

\bibitem[Kingma and Ba(2014)]{kingma2014adam}
D.~P. Kingma and J.~Ba.
\newblock Adam: A method for stochastic optimization.
\newblock \emph{arXiv preprint arXiv:1412.6980}, 2014.

\bibitem[Fan and Mitra(2015)]{fan2015bounded}
C.~Fan and S.~Mitra.
\newblock Bounded verification with on-the-fly discrepancy computation.
\newblock \emph{arXiv preprint arXiv:1502.01801}, 2015.

\bibitem[Lei et~al.(2018)Lei, G’Sell, Rinaldo, Tibshirani, and
  Wasserman]{lei2018distribution}
J.~Lei, M.~G’Sell, A.~Rinaldo, R.~J. Tibshirani, and L.~Wasserman.
\newblock Distribution-free predictive inference for regression.
\newblock \emph{Journal of the American Statistical Association}, 2018.

\bibitem[Khalil and Grizzle(2002)]{khalil2002nonlinear}
H.~K. Khalil and J.~W. Grizzle.
\newblock \emph{Nonlinear systems}, volume~3.
\newblock Prentice hall Upper Saddle River, NJ, 2002.

\bibitem[Yang and Jagannathan(2011)]{yang2011reinforcement}
Q.~Yang and S.~Jagannathan.
\newblock Reinforcement learning controller design for affine nonlinear
  discrete-time systems using online approximators.
\newblock \emph{IEEE Transactions on Systems, Man, and Cybernetics, Part B
  (Cybernetics)}, 42\penalty0 (2):\penalty0 377--390, 2011.

\bibitem[Janson et~al.(2015)Janson, Schmerling, Clark, and
  Pavone]{janson2015fast}
L.~Janson, E.~Schmerling, A.~Clark, and M.~Pavone.
\newblock Fast marching tree: A fast marching sampling-based method for optimal
  motion planning in many dimensions.
\newblock \emph{The International journal of robotics research}, 34\penalty0
  (7):\penalty0 883--921, 2015.

\bibitem[Richter et~al.(2016)Richter, Bry, and Roy]{richter2016polynomial}
C.~Richter, A.~Bry, and N.~Roy.
\newblock Polynomial trajectory planning for aggressive quadrotor flight in
  dense indoor environments.
\newblock In \emph{Robotics Research}, pages 649--666. Springer, 2016.

\bibitem[Karaman et~al.(2011)Karaman, Walter, Perez, Frazzoli, and
  Teller]{karaman2011anytime}
S.~Karaman, M.~R. Walter, A.~Perez, E.~Frazzoli, and S.~Teller.
\newblock Anytime motion planning using the rrt.
\newblock In \emph{2011 IEEE International Conference on Robotics and
  Automation}, pages 1478--1483. IEEE, 2011.

\bibitem[Kobilarov(2012)]{kobilarov2012cross}
M.~Kobilarov.
\newblock Cross-entropy motion planning.
\newblock \emph{The International Journal of Robotics Research}, 31\penalty0
  (7):\penalty0 855--871, 2012.

\bibitem[Tomlin et~al.(2003)Tomlin, Mitchell, Bayen, and
  Oishi]{tomlin2003computational}
C.~J. Tomlin, I.~Mitchell, A.~M. Bayen, and M.~Oishi.
\newblock Computational techniques for the verification of hybrid systems.
\newblock \emph{Proceedings of the IEEE}, 91\penalty0 (7):\penalty0 986--1001,
  2003.

\bibitem[Gillula et~al.(2010)Gillula, Huang, Vitus, and
  Tomlin]{gillula2010design}
J.~H. Gillula, H.~Huang, M.~P. Vitus, and C.~J. Tomlin.
\newblock Design of guaranteed safe maneuvers using reachable sets: Autonomous
  quadrotor aerobatics in theory and practice.
\newblock In \emph{2010 IEEE International Conference on Robotics and
  Automation}, pages 1649--1654. IEEE, 2010.

\bibitem[Huang et~al.(2011)Huang, Ding, Zhang, and
  Tomlin]{huang2011differential}
H.~Huang, J.~Ding, W.~Zhang, and C.~J. Tomlin.
\newblock A differential game approach to planning in adversarial scenarios: A
  case study on capture-the-flag.
\newblock In \emph{2011 IEEE International Conference on Robotics and
  Automation}, pages 1451--1456. IEEE, 2011.

\bibitem[Fridovich-Keil et~al.(2018)Fridovich-Keil, Herbert, Fisac, Deglurkar,
  and Tomlin]{fridovich2018planning}
D.~Fridovich-Keil, S.~L. Herbert, J.~F. Fisac, S.~Deglurkar, and C.~J. Tomlin.
\newblock Planning, fast and slow: A framework for adaptive real-time safe
  trajectory planning.
\newblock In \emph{2018 IEEE International Conference on Robotics and
  Automation (ICRA)}, pages 387--394. IEEE, 2018.

\bibitem[Chen et~al.(2018)Chen, Herbert, Vashishtha, Bansal, and
  Tomlin]{chen2018decomposition}
M.~Chen, S.~L. Herbert, M.~S. Vashishtha, S.~Bansal, and C.~J. Tomlin.
\newblock Decomposition of reachable sets and tubes for a class of nonlinear
  systems.
\newblock \emph{IEEE Transactions on Automatic Control}, 63\penalty0
  (11):\penalty0 3675--3688, 2018.

\bibitem[Prajna et~al.(2007)Prajna, Jadbabaie, and Pappas]{prajna2007framework}
S.~Prajna, A.~Jadbabaie, and G.~J. Pappas.
\newblock A framework for worst-case and stochastic safety verification using
  barrier certificates.
\newblock \emph{IEEE Transactions on Automatic Control}, 52\penalty0
  (8):\penalty0 1415--1428, 2007.

\bibitem[Barry et~al.(2012)Barry, Majumdar, and Tedrake]{barry2012safety}
A.~J. Barry, A.~Majumdar, and R.~Tedrake.
\newblock Safety verification of reactive controllers for uav flight in
  cluttered environments using barrier certificates.
\newblock In \emph{2012 IEEE International Conference on Robotics and
  Automation}, pages 484--490. IEEE, 2012.

\bibitem[Majumdar and Tedrake(2013)]{majumdar2013robust}
A.~Majumdar and R.~Tedrake.
\newblock Robust online motion planning with regions of finite time invariance.
\newblock In \emph{Algorithmic foundations of robotics X}, pages 543--558.
  Springer, 2013.

\bibitem[Langson et~al.(2004)Langson, Chryssochoos, Rakovi{\'c}, and
  Mayne]{langson2004robust}
W.~Langson, I.~Chryssochoos, S.~Rakovi{\'c}, and D.~Q. Mayne.
\newblock Robust model predictive control using tubes.
\newblock \emph{Automatica}, 40\penalty0 (1):\penalty0 125--133, 2004.

\bibitem[Mayne et~al.(2005)Mayne, Seron, and Rakovi{\'c}]{mayne2005robust}
D.~Q. Mayne, M.~M. Seron, and S.~Rakovi{\'c}.
\newblock Robust model predictive control of constrained linear systems with
  bounded disturbances.
\newblock \emph{Automatica}, 41\penalty0 (2):\penalty0 219--224, 2005.

\bibitem[Farina and Scattolini(2012)]{farina2012tube}
M.~Farina and R.~Scattolini.
\newblock Tube-based robust sampled-data mpc for linear continuous-time
  systems.
\newblock \emph{Automatica}, 48\penalty0 (7):\penalty0 1473--1476, 2012.

\bibitem[K{\"o}gel and Findeisen(2015)]{kogel2015discrete}
M.~K{\"o}gel and R.~Findeisen.
\newblock Discrete-time robust model predictive control for continuous-time
  nonlinear systems.
\newblock In \emph{2015 American Control Conference (ACC)}, pages 924--930.
  IEEE, 2015.

\bibitem[Yu et~al.(2013)Yu, Maier, Chen, and Allg{\"o}wer]{yu2013tube}
S.~Yu, C.~Maier, H.~Chen, and F.~Allg{\"o}wer.
\newblock Tube mpc scheme based on robust control invariant set with
  application to lipschitz nonlinear systems.
\newblock \emph{Systems \& Control Letters}, 62\penalty0 (2):\penalty0
  194--200, 2013.

\bibitem[Mayne et~al.(2011)Mayne, Kerrigan, Van~Wyk, and Falugi]{mayne2011tube}
D.~Q. Mayne, E.~C. Kerrigan, E.~Van~Wyk, and P.~Falugi.
\newblock Tube-based robust nonlinear model predictive control.
\newblock \emph{International Journal of Robust and Nonlinear Control},
  21\penalty0 (11):\penalty0 1341--1353, 2011.

\bibitem[Aylward et~al.(2008)Aylward, Parrilo, and
  Slotine]{aylward2008stability}
E.~M. Aylward, P.~A. Parrilo, and J.-J.~E. Slotine.
\newblock Stability and robustness analysis of nonlinear systems via
  contraction metrics and sos programming.
\newblock \emph{Automatica}, 44\penalty0 (8):\penalty0 2163--2170, 2008.

\bibitem[Leung and Manchester(2017)]{leung2017nonlinear}
K.~Leung and I.~R. Manchester.
\newblock Nonlinear stabilization via control contraction metrics: A
  pseudospectral approach for computing geodesics.
\newblock In \emph{2017 American Control Conference (ACC)}, pages 1284--1289.
  IEEE, 2017.

\end{thebibliography}
\begin{center}
    \Large{Appendix}
\end{center}

\section{Justification for the CCM regularization terms}
In theory, minimizing $\mathcal{L}_u$ is enough for learning a tracking controller, since inequality~\eqref{eq:c3} implies the contraction of the closed-loop system. However, in the experiments, we found that minimizing $\mathcal{L}_u$ solely is hard. The training loss stayed on a relatively high level, and the accuracy of inequality~\eqref{eq:c3} was extremely low ($\sim 1\%$). Here, $\mathcal{L}_{w1}$ and $\mathcal{L}_{w2}$ are used as auxiliary loss terms to mitigate the difficulty of minimizing $\mathcal{L}_u$. After adding these terms, the accuracy approached $\sim 100\%$.

Furthermore, conditions~(2)~and~(3) are sufficient but not necessary for the existence of contracting controller. In the derivation of conditions~(2)~and~(3), the inequality (inequality (4) in the main paper) which indicates contracting is required to be hold for all unbounded $u$, which is not necessary in the sense that $u$ is usually bounded in a control problem. Thus, even if conditions~(2)~and~(3) do not exactly hold, a contracting tracking controller may still exist.

All the benchmark models used in experiments are under-actuated where the rank of the control matrix $r(B) < n$. In cases where $r(B) = n$ for all $x$ and thus there does not exist a non-zero annihilator matrix $B_\bot$ such that $B_\bot^\T B = 0$, one can easily verify that any uniformly positive definite metric function $M(x)$ can be a valid CCM. The control synthesis is also straightforward in this case. One can directly compute the instantaneous control $u$ given the desired $\dot{x}$ since $B(x)$ is invertible. Moreover, as shown in~\citep{manchester2017control}, existence of a CCM is guaranteed for feedback linearizable systems. In that case, a valid CCM can be constructed analytically from the variable and feedback transformation. We refer the interested readers to~\citep{manchester2017control} for more details.

\section{Proof of Proposition~\ref{thm:lconstant} and CCM validation process example}
\subsection{Proof of Proposition~\ref{thm:lconstant}}
We denote the operator $2$-norm for matrices by $\|\cdot\|_2$. The following lemmas are going to be used for the proof of Proposition~\ref{thm:lconstant}.

\begin{lemma}\label{lemma:eigdiff}
For any two symmetric matrices $A,B \in \reals^{n \times n}$, the difference of their eigenvalues satisfies:
\[
| \lambda_{\max}(A) - \lambda_{\max}(B)| \leq \|A - B \|_2.
\]
\end{lemma}

Lemma~\ref{lemma:eigdiff} is a well-known result that uses the Courant-Fischer
minimax theorem. The detailed proof can be found at~\cite{fan2015bounded}.

\begin{lemma}
\label{lemma:Lmult}
For any two functions $A: \reals^{n} \rightarrow \reals^{k \times \ell}$ and $B: \reals^{n} \rightarrow \reals^{\ell \times s}$, if $A$ and $B$ both have Lipschitz constants $L_A$ and $L_B$ with respect to $2$-norm respectively, moreover, if $\forall x, \|A(x)\|_2 \leq S_A$ and $\| B(x) \|_2 \leq S_B$, then for any $x,y \in \reals^{n}$,
\[
\| A(x)B(x) - A(y)B(y) \|_2 \leq \left( S_A L_B + S_B L_A \right) \| x - y \|_2. 
\]
\end{lemma}
\begin{proof}
\begin{equation}
    \begin{array}{rcl}
        \| A(x)B(x) - A(y)B(y) \|_2 & = & \| A(x)B(x) - A(x)B(y) + A(x) B(y) - A(y)B(y) \|_2 \\
                                    & \leq & L_B \|A(x)\|_2\|x-y\|_2 + L_A \|B(y)\|_2 \|x-y\|_2 \\
                                    & \leq & \left( S_A L_B + S_B L_A \right) \| x - y \|_2.
    \end{array}
\end{equation}
\end{proof}

The proof of Proposition~\ref{thm:lconstant} follows from the above lemmas.
\begin{proof}
For all $X$ and $Y \in \X\times\X\times\U$, we have
\begin{align*}
    \big|\lambda_{\max}\left( \dot{M}(X) + (A(X)+B(X)K(X))^\T M(X) + M(X) (A(X)+B(X)K(X)) + 2 \lambda M(X) \right)   \\
    - \lambda_{\max}\left( \dot{M}(Y) + (A(Y)+B(Y)K(Y))^\T M(Y) + M(Y) (A(Y)+B(Y)K(Y)) + 2 \lambda M(Y) \right)\big| \\
    \leq \|\dot{M}(X) - \dot{M}(Y) \|_2 +  2\|M(X)A(X)-M(Y)A(Y)\|_2 \\
     + 2\|M(X)B(X)K(X) - M(Y)B(Y)K(Y)\|_2 + 2\lambda\|M(X) - M(Y)\|_2  \\
    \leq \left(L_{\dot{M}} + 2 \left(S_ML_A + S_A L_M + S_MS_BL_K +  S_BS_KL_M + S_M S_KL_B + \lambda L_M \right)\right) \|X-Y\|_2.
\end{align*}
\end{proof}

\subsection{CCM validation on Dubin's vehicle}
First, the Lipschitz constants and bounds on $2$-norms are derived as follows.
\paragraph{$L_A$ and $S_A$:} By definition, $A = \frac{\partial f}{\partial x} + \sum\frac{\partial b_i}{\partial x} u^i = \begin{bmatrix} 0 & 0 & -v\cos(\theta) & \cos(\theta) \\ 0 & 0 & v\sin(\theta) & \sin(\theta) \\ 0 & 0 & 0 & 0 \\ 0 & 0 & 0 & 0 \end{bmatrix}$. Considering a vector-valued function $\tilde{A} = [-v\cos(\theta), v\sin(\theta), \cos(\theta), \sin(\theta)]$ and its Lipschitz constant $L_{\tilde{A}}$ w.r.t. the Frobenius norm, since the Jacobian of $\tilde{A}$ is given by $\frac{\partial \tilde{A}}{\partial x} = \begin{bmatrix} -\cos(\theta) & v\sin(\theta) \\ \sin(\theta) & v\cos(\theta) \\ 0 & -\sin(\theta) \\ 0 & \cos(\theta) \end{bmatrix}$, we have $L_{\tilde{A}} = \sup ||\frac{\partial \tilde{A}}{\partial x}||_F = \sqrt{2+v_h^2}$, where $v_h$ is the upper bound of the velocity. Obviously, under the Frobenius norm, the Lipschitz constants of $A$ and $\tilde{A}$ coincide. Moreover, for any matrix $A$, we have $\|A\|_2 \leq \|A\|_F$, and thus $L_A \leq L_{\tilde{A}} = \sqrt{2+v_h^2}$. As for the norm bound, $S_A = \sup \|A\|_2 \leq \sup \|A\|_F = \sqrt{1+v_h^2}$.

\paragraph{$L_B$ and $S_B$:} By definition, $B = \begin{bmatrix} 0 & 0 \\ 0 & 0 \\ 1 & 0 \\ 0 & 1 \end{bmatrix}$, which is a constant. Thus, $L_B = 0$, and $S_B = \sigma_{max}(B) = 1$, where $\sigma_{max}$ is the largest singular value.

\paragraph{$L_M$ and $S_M$:} By definition, we have $W(x) = C(x)^\T C(x) + \underline{w}\mathbf{I}$ and $M(x) = W(x)^{-1}$, and thus $S_M = \sup \|M(x)\|_2 \leq \frac{1}{\underline{w}}$. Moreover, $\|M(x) - M(y)\|_2 = \|M(x) (W(x) - W(y)) M(y)\|_2 \leq \frac{1}{\underline{w}^2} \|W(x) - W(y)\|_2$. Thus, we have $L_M \leq \frac{1}{\underline{w}^2} L_W$. As for $L_W$, by Lemma~\ref{lemma:Lmult}, $L_W \leq 2S_C L_C$, where $L_C$ can be evaluated using tools for estimating the Lipschitz constants of NNs. In our experiments, we use the official implementation\footnote{\url{https://github.com/arobey1/LipSDP}} of~\citep{fazlyab2019efficient}. After obtaining $L_C$, $S_C$ can be upper bounded using samples and the Lipschitz constant $L_C$.

\paragraph{$L_{\dot{M}}$:} By definition, $\dot{M} = \sum_{i=1}^{n} \frac{\partial M}{\partial x_i} \dot{x_i}$. We randomly sampled 10K points and evaluate the gradient of $\dot{M}$, then the maximum of the norm of the gradient is returned as a Lipschitz constant.

\paragraph{$L_K$ and $S_K$:} For simplicity of the analysis, we use a simpler neural controller $u = k(x,x^*) x_e + u^*$, where $k$ is a two-layer neural network and $x_e := (x-x^*)$. Then, $K = \frac{\partial u}{\partial x} = k(x,x^*) + \sum_{i=1}^{n} x_e^i \frac{\partial k_i}{\partial x}$, where $k_i$ is the $i$-th column of $k$ and $x_e^i$ is the $i$-th element of $x_e$. Thus, $S_K \leq S_k + \sum_{i=1}^{n} \sup |x_e^i| L_{k_i|x}$, where $L_{k_i|x}$ is the Lipschitz constant of $k_i$ w.r.t. $x$. Also, $L_K \leq L_k + \sum_{i=1}^{n} \sup |x_e^i| L_{\frac{\partial ki}{\partial x}} + \sqrt{2} \sum_{i=1}^n L_{k_i|x}$, where we use the facts that $x_{e}^i = x_i - x^*_i \implies L_{x_e^i} = \sqrt{2}$, and $S_{\frac{\partial k_i}{\partial x}} \leq L_{k_i|x}$. As for $L_{\frac{\partial k_i}{\partial x}}$, we have $\frac{\partial k_i}{\partial x} = A_2^i \cdot diag(dtanh(A_1 [x, x^*])) \cdot A_{11}$, where $A_1$ is the weights in the first layer, $A_{11}$ is the weights associated to $x$ in the first layer, $A_2^i$ is the weights associated to the $i$-th column of $k$ in the second layer, and $dtanh$ is the derivative of the $tanh$ function. Thus, $L_{\frac{\partial k_i}{\partial x}} \leq ||A_2^i|| \cdot ||A_{11}|| \cdot ||A_1|| \cdot L_{dthan}$.

\paragraph{Validation process:}
We successfully verified the trained model for a state space where $[-5, -5, -\pi, 1] \leq x,x^* \leq [5, 5, \pi, 2]$, $[-0.1, -0.1, -0.1, -0.1] \leq x-x^* \leq [0.1,0.1,0.1,0.1]$, and $[-1,0] \leq u^* \leq [1,0]$ with $\lambda = 0.1$. The derived Lipschitz constant using Proposition~\ref{thm:lconstant} for the trained model was $616.8$. After that, we discretized the state space and evaluated the maximum eigenvalues of the LHS of inequality~\eqref{eq:c3} at the grid points, and the results gave a theoretical guarantee on the satisfaction of inequality~\eqref{eq:c3} for all $x$, $x^*$, and $u^*$.

\section{Proof of Proposition~\ref{prop:quantile}}
The following lemma from~\cite{lei2018distribution} is going to be used.
\begin{lemma}
Consider a set of i.i.d. random variables $\{U_1, \cdots, U_n\}$, then the following inequality holds for a new i.i.d. sample $U_{n+1}$,
\[
\Pr(U_{n+1} \geq q_{1-\alpha}) \leq 1 - \alpha,
\]
where $q_{1-\alpha}$ is $(1-\alpha)$-th quantile of $\{U_i\}_{i=1}^{n}$, i.e.
\[
q_{1-\alpha} = \begin{cases} U_{(\lceil (1-\alpha)(n+1) \rceil)} &\mbox{if } \lceil (1-\alpha)(n+1) \rceil \leq n, \\
\infty &\mbox{otherwise} \end{cases},
\]
where $U_{(1)} \leq \cdots \leq U_{(n)}$ is the order statistics of $\{U_{i}\}_{i=1}^{n}$.
\end{lemma}
The above lemma follows from an obvious fact that the rank of $U_{n+1}$ in $\{U_i\}_{i=1}^{n+1}$ is uniformly distributed in $\{1,\cdots,n+1\}$.

Since the evaluation configurations are all i.i.d., the corresponding quality scores $\{m_i\}_{i=1}^{n+1}$ are also i.i.d. The proposition immediately follows from the above lemma.

\section{Proof of Theorem~\ref{thm:robustness}}
\begin{proof}
Since $\forall x$, $M(x) \in S_n^{>0}$, there exists $\Theta(x)$ such that $M(x) = \Theta(x)^\T \Theta(x)$ for all $x$. Now, let us consider the smallest path integral between $x_1(t)$ and $x_2(t)$ w.r.t. metric $M(x)$ and denote it by $R(t) = \int_{x_1(t)}^{x_2(t)} \sqrt{\dx^\T M(x) \dx} = \int_{x_1(t)}^{x_2(t)} \|\Theta(x) \dx\|_2$. Then differentiating $R(t)$ w.r.t. $t$ yields $\dot{R} \leq -\lambda R + \|\Theta d\|_2 \leq -\lambda R + \epsilon \sqrt{\overline{m}}$. (see~\citep{lohmiller1998contraction}, pp. 11 (vii)) By comparison lemma~(see \citep{khalil2002nonlinear}, pp. 102), $R(t)$ is bounded as $R(t) \leq R(0) e^{- \lambda t} + \frac{\epsilon}{\lambda} \sqrt{\overline{m}} (1 - e^{-\lambda t})$. Since $\Theta$ is uniformly bounded as $\Theta(x) \succeq \sqrt{\underline{m}} \mathbf{I}$, we have $R(t) \geq \sqrt{\underline{m}} \|x_1(t) - x_2(t)\|_2$, which follows from the fact that the smallest path integral w.r.t. a constant metric is the path integral along the straight line. Combining these two inequalities yields the theorem.
\end{proof}

\section{More details on the experimental results}
\subsection{Hyper-parameters}
For all the models, we set $\underline{w} = 0.1$ and $\overline{w} = 10$. For PVTOL, Quadrotor, Neural lander, Segway, Quadrotor2, and Cart-pole, we set $\lambda = 0.5$. For Dubin's car, we set $\lambda = 1$. For TLPRA, we set $\lambda = 2$. For Pendulum, we set $\lambda = 3$.

For the evaluation of RL, we trained the model for 1K epochs (1K steps per epoch, and thus 1M steps in total). The best model was used for evaluation.

\subsection{Dynamics of the $4$ benchmarks discussed in Section~\ref{sec:exps}}
As mentioned in the main paper, for each benchmark model, we have to define the state space $\X$ and the reference control space $\U$. Also, for evaluation purpose, the initial state of the reference trajectory is sampled from a set $\X_0$, and the initial error is sampled from a set $\X_{e0}$. All these sets are defined as hyper-rectangles. For two vectors $x,y \in \reals^n$, $x \leq y$ or $x \geq y$ means element-wise inequality.
\paragraph{PVTOL} models a planar vertical-takeoff-vertical-landing (PVTOL) system for drones and is adopted from~\citep{singh2019robust, singh2019learning}. The system includes $6$ variables $x := [p_x,p_z,\phi,v_x,v_z,\dot{\phi}]$, where $p_x$ and $p_z$ are the positions, $v_x$ and $v_z$ are the corresponding velocities, $\phi$ is the roll, and $\dot{\phi}$ is the angular rate. The $2$-dimensional control input $u$ corresponds to the thrusts of the two rotors. The dynamics of the system are as follows:
\[\dot{x} = \begin{bmatrix}v_x \cos(\phi) - v_z \sin(\phi)\\ v_x \sin(\phi) + v_z \cos(\phi)\\ \dot{\phi}\\ v_z \dot{\phi} - g\sin(\phi)\\ -v_x \dot{\phi} - g\cos(\phi)\\ 0 \end{bmatrix} + \begin{bmatrix}0 & 0\\ 0 & 0\\ 0 & 0\\ 0 & 0\\ 1/m & 1/m\\ l/J & -l/J \end{bmatrix} u,\] where $m=0.486$, $J = 0.00383$, $g = 9.81$, and $l = 0.25$. For the sets, we set $\X = \{[-35, -2, -\frac{\pi}{3}, -2, -1, -\frac{\pi}{3}] \leq x \leq [0, 2, \frac{\pi}{3}, 2, 1, \frac{\pi}{3}]\}$, $\U = \{[mg/2 - 1, mg/2 - 1] \leq u \leq [mg/2 + 1, mg/2 + 1]\}$, $\X_0 = \{[0, 0, -0.1, 0.5, 0, 0] \leq x_0 \leq [0, 0,  0.1,  1, 0, 0]\}$, and $\X_{e0} = \{[-0.5, -0.5, -0.5, -0.5, -0.5, -0.5] \leq x_{e0} \leq [0.5, 0.5, 0.5, 0.5, 0.5, 0.5]\}$.
\paragraph{Quadrotor} is adopted from~\citep{singh2019robust}. The state variables are given by $x := [p_x, p_y, p_z, v_x, v_y, v_z, f, \phi, \theta, \psi]$, where $p_k$ and $v_k$ for $k\in\{x,y,z\}$ are positions and velocities, and $\phi$, $\theta$ and $\psi$ are roll, pitch, and yaw respectively, and $f > 0$ is the net (normalized by mass) thrust generated by the four rotors. The control input is considered as $u := [\dot{f}, \dot{\phi}, \dot{\theta}, \dot{\psi}]$. The dynamics of the system are as follows:
\[\dot{x} = \begin{bmatrix} v_x\\v_y\\v_z\\-f\sin(\theta)\\f\cos(\theta)\sin(\phi)\\g-f\cos(\theta)\cos(\phi)\\0\\0\\0\\0\end{bmatrix} + \begin{bmatrix} 0 & 0 & 0 & 0\\0 & 0 & 0 & 0\\0 & 0 & 0 & 0\\0 & 0 & 0 & 0\\0 & 0 & 0 & 0\\0 & 0 & 0 & 0\\1 & 0 & 0 & 0\\0 & 1 & 0 & 0\\0 & 0 & 1 & 0\\0 & 0 & 0 & 1\\ \end{bmatrix} u,\] where $g = 9.81$. For the sets, we set $\X = \{[-30, -30, -30, -1.5, -1.5, -1.5, 0.5g, -\frac{\pi}{3}, -\frac{\pi}{3}, -\frac{\pi}{3}] \leq x \leq [30, 30, 30, 1.5, 1.5, 1.5, 2g, \frac{\pi}{3}, \frac{\pi}{3}, \frac{\pi}{3}]$, $\U = \{[-1, -1, -1, -1] \leq u \leq [1, 1, 1, 1]\}$, $\X_0 = \{[-5, -5, -5, -1, -1, -1, g, 0, 0, 0] \leq x_0 \leq [5,  5,  5,  1,  1,  1, g, 0, 0, 0]\}$, and $\X_{e0} = \{[-0.5, -0.5, -0.5, -0.5, -0.5, -0.5, -0.5, -0.5, -0.5, -0.5] \leq x_{e0} \leq [0.5, 0.5, 0.5, 0.5, 0.5, 0.5, 0.5, 0.5, 0.5, 0.5]\}$.
\paragraph{Neural lander} models a drone flying close to the ground so that ground effect is prominent and therefore could not be ignored~\citep{liu2019robust}. The state is give by $x := [p_x,p_y,p_z, v_x, v_y, v_z]$, where $p_k$ and $v_k$ for $k \in \{x, y, z\}$ are the positions and velocities. The $3$-dimensional control input is given by $u := [a_x, a_y, a_z]$ which are the acceleration in three directions. Note that in this model the ground effect is learned from empirical data of a physical drone using a $4$-layer neural network~\citep{liu2019robust}. Due to the neural network function in the dynamics, the SoS-based method as in~\citep{singh2019robust} cannot handle such a model since neural networks are hard to be approximated by polynomials. However, C3M can handle this model since it doesn't impose hard constraints and only requires the ability to evaluate the dynamics and its gradient. The dynamics of the system are as follows:
\[\dot{x} = \begin{bmatrix}v_x\\v_y\\v_z\\Fa_1/m\\Fa_2/m\\Fa_3/m - g\end{bmatrix} + \begin{bmatrix}0&0&0\\0&0&0\\0&0&0\\1&0&0\\0&1&0\\0&0&1\\\end{bmatrix} u,\] where $m = 1.47$, $g = 9.81$, and $Fa_i = Fa_i(z,v_x,v_y.v_z)$ for $i=1,2,3$, are neural networks. For the sets, we set $\X = \{[-5, -5, 0, -1, -1, -1] \leq x \leq [5, 5, 2, 1, 1, 1]$, $\U = \{[-1,-1,-3] \leq u \leq [1,1,9]\}$, $\X_0 = \{[-3, -3, 0.5, 1, 0, 0] \leq x_0 \leq [3,  3, 1, 1, 0, 0]\}$, and $\X_{e0} = \{[-1, -1, -0.4, -1, -1, 0] \leq x_{e0} \leq [1, 1, 1, 1, 1, 0]\}$.
\paragraph{SEGWAY} models a real-world Segway robot and is adopted from \citep{taylor2019learning}. The state variables are given by $x := [p_x, \theta, v_x, \omega]$ where $p_x$ and $v_x$ are the position and velocity, and $\theta$ and $\omega$ are the angle and angular velocity. The control input is the torque of the motor. The dynamics of the system are as follows:
\[\dot{x} = \begin{bmatrix}v_x\\\omega\\\frac{\cos(\theta)(9.8\sin(\theta) - 1.8u + 11.5v) - 10.9u + 68.4v - 1.2 \omega^2 \sin(\theta)}{\cos(\theta) - 24.7}\\\frac{(9.3u-58.8v)\cos(\theta) + 38.6u - 243.5v - \sin(\theta)(208.3+\omega^2\cos(\theta))}{\cos^2(\theta) - 24.7}\end{bmatrix}.\] For the sets, we set $\X = \{[-5, -\frac{\pi}{3}, -1, -\pi] \leq x \leq [5, \frac{\pi}{3}, 1, \pi]$, $\U = \{0 \leq u \leq 0\}$, $\X_0 = \{[0,0,0,0] \leq x_0 \leq [0,0,0,0]\}$, and $\X_{e0} = \{[-1, -\frac{\pi}{3}, -0.5, -\pi] \leq x_{e0} \leq [1, \frac{\pi}{3}, 0.5, \pi]\}$.

\subsection{Additional benchmarks}
\paragraph{Cart-pole} is a well-known model for evaluating control algorithms. The state variables are given by $x := [p_x, \theta, v_x, \omega]$. The dynamics are given by
\[
\dot{x} = \begin{bmatrix}
v_x\\
\omega\\
\frac{u+m_p\sin(\theta)(l\omega^2-g\cos(\theta))}{m_c + m_p\sin^2(\theta)}\\
\frac{u\cos(\theta) + m_pl\omega^2\cos(\theta)\sin(\theta) - (m_c + m_p)g\sin(\theta)}{l(m_c + m_p\sin^2(\theta))}
\end{bmatrix},
\]
where $m_c = 1$, $m_p = 1$, $g = 9.8$, and $l = 1$. For the sets, we set $\X = \{[-5, -\frac{\pi}{3}, -1, -1] \leq x \leq [5, \frac{\pi}{3}, 1, 1]$, $\U = \{0 \leq u \leq 0\}$, $\X_0 = \{[0,0,0,0] \leq x_0 \leq [0,0,0,0]\}$, and $\X_{e0} = \{[-0.3, -0.3, -0.3, -0.3] \leq x_{e0} \leq [0.3, 0.3, 0.3, 0.3]\}$.

\paragraph{Quadrotor2} is adopted from~\citep{herbert2017fastrack}. The state variables are given by $x := [p_x, p_y, p_z, v_x, v_y, \theta_x, \theta_y, \omega_x, \omega_y, v_z]$. The control input is given by $u:=[a_x,a_y,a_z]$. The dynamics are given by
\[
\dot{x} = \begin{bmatrix}v_x\\v_y\\v_z\\g\tan(\theta_x)\\g\tan(\theta_y)\\-d_1\theta_x+\omega_x\\-d_1\theta_y+\omega_y\\-d_0\theta_x+n_0a_x\\-d_0\theta_y+n_0a_y\\k_Ta_z - g\end{bmatrix},
\] where $d_0 = 10$, $d_1 = 8$, $n_0 = 10$, $k_T = 0.91$, and $g = 9.81$. For the sets, we set $\X = \{[-15, -15, -15, -2, -2, -\frac{\pi}{3}, -\frac{\pi}{3}, -\frac{\pi}{3}, -\frac{\pi}{3}, -2] \leq x \leq [15, 15, 15, 2, 2, \frac{\pi}{3}, \frac{\pi}{3}, \frac{\pi}{3}, \frac{\pi}{3}, 2]$, $\U = \{[-10,-10,0] \leq u \leq [10,10,1.5g]\}$, $\X_0 = \{[-2,-2,-2,-1,-1,-0.5,-0.5,-0.5,-0.5,-1] \leq x_0 \leq [2,2,2,1,1,0.5,0.5,0.5,0.5,1]\}$, and $\X_{e0} = \{[-0.5, -0.5, -0.5, -0.5, -0.5, -0.5, -0.5, -0.5, -0.5, -0.5] \leq x_{e0} \leq [0.5, 0.5, 0.5, 0.5, 0.5, 0.5, 0.5, 0.5, 0.5, 0.5]\}$.
\paragraph{Pendulum} is a simple model for evaluating control methods. The state variables are given by $x := [\theta, \dot{\theta}]$, where $\theta$ is the angle, and $\dot{\theta}$ is the angular velocity. The one-dimensional control input corresponds to the torque of the motor. The dynamics are as follows:
\[
\dot{x} = \begin{bmatrix}\dot{\theta}\\\frac{mgl\sin(\theta) - 0.1\dot{\theta} + u}{ml^2}\end{bmatrix},
\] where $g = 9.81$, $m = 0.15$, and $l = 0.5$. As for the sets, we set $\X = \{[-\frac{\pi}{3}, -\frac{\pi}{3}] \leq x \leq [\frac{\pi}{3}, \frac{\pi}{3}]$, $\U = \{-1 \leq u \leq 1\}$, $\X_0 = \{[0,0] \leq x_0 \leq [0,0]\}$, and $\X_{e0} = \{[-\frac{\pi}{4}, -\frac{\pi}{4}] \leq x_{e0} \leq [\frac{\pi}{4}, \frac{\pi}{4}]\}$.
\paragraph{Dubin's car.} The state of the car is given by $x := [p_x,p_y,\theta,v]$, where $p_x, p_y$ are the positions, $\theta$ is the heading direction, and $v$ is the velocity. Its motion is controlled by two inputs: angular velocity $\omega$ and linear acceleration $a$. The car's dynamics are given by:
\[
\dot{x} = \begin{bmatrix}v\cos(\theta)\\v\sin(\theta)\\\omega\\a\end{bmatrix}.\] For the sets, we set $\X = \{[-5, -5, -\pi, 1] \leq x \leq [5, 5, \pi, 2]$, $\U = \{[-1,0] \leq u \leq [1,0]\}$, $\X_0 = \{[-2, -2, -1, 1.5] \leq x_0 \leq [2, 2, 1, 1.5]\}$, and $\X_{e0} = \{[-1, -1, -1, -1] \leq x_{e0} \leq [1,1,1,1]\}$. 
There is a simpler Dubin's car model with three state variables $x := [p_x,p_y,\theta]$ and a two-dimensional control input $u := [v, \omega]$. The dynamics are similar to the above ODEs without $\dot{v} = a$. However, we notice that in the simpler model $f(x) \equiv 0$ and $B(x)$ is not full-rank. Therefore, Eq.~\eqref{eq:c1} in the paper is not feasible and the loss term $\mathcal{L}_{w1}$ does not make sense. If we train the model only with the loss $\mathcal{L}_{u}$ as in Eq.~\eqref{eq:loss_controller}, the resulting controller is not performing well.

\paragraph{TLPRA} models a Two-Link Planar Robot Arm system and is adopted from~\citep{yang2011reinforcement}. The state variables are given by $x := [q_1, q_2, \dot{q}_1, \dot{q}_2]$, where $q$ and $\dot{q}$ are the rotational angle and angular velocity of the joints respectively. The control input is given by $u := [\tau_1, \tau_2]$ which are the torques applied on the joints. The dynamics are given by
\begin{multline*}
\begin{bmatrix}\ddot{q}_1\\\ddot{q}_2\end{bmatrix} = \begin{bmatrix}\alpha+\beta+2\eta \cos(q_2) & \beta + \eta \cos(q_2)\\\beta + \eta \cos(q_2) & \beta \end{bmatrix}^{-1} 
\\
\begin{bmatrix}
\tau_1 + \eta(2\dot{q}_1\dot{q}_2 + \dot{q}_2^2)\sin(q_2) - \alpha e_1 \cos(q_1) - \eta e_1 \cos(q_1 + q_2)\\\tau_2 - \eta \dot{q}_1^2 \sin(q_2) - \eta e_1 \cos(q_1+q_2)\end{bmatrix},
\end{multline*} where $\alpha = (m_1 + m_2)a_1^2$, $\beta = m_2a_2^2$, $\eta = m_2a_1a_2$, $e_1 = g/a_1$, $g = 9.8$, $m_1 = 0.8$, $m_2 = 2.3$, $a_1 = 1$, and $a_2 = 1$. For the sets, we set $\X = \{[-\frac{\pi}{2}, -\frac{\pi}{2}, -\frac{\pi}{3}, -\frac{\pi}{3}] \leq x \leq [\frac{\pi}{2}, \frac{\pi}{2}, \frac{\pi}{3}, \frac{\pi}{3}]$, $\U = \{[0,0] \leq u \leq [0,0]\}$, $\X_0 = \{[\frac{\pi}{2}, 0, 0, 0] \leq x_0 \leq [\frac{\pi}{2}, 0, 0, 0]\}$, and $\X_{e0} = \{[-0.3,-0.3,0,0] \leq x_{e0} \leq [0.3,0.3,0,0]\}$.

\subsection{Experimental results on additional benchmarks}
Results comparing the performance of \toolname and other approaches on the additional $5$ models are shown in Fig.~\ref{fig:results_more} and Tab.~\ref{tab:results_more}. On some simple models including Pendulum, Dubin's car, Cart-pole, and TLPRA, both RL and \toolname achieve perfect results, and \toolname demonstrates lower variance. On the higher-dimensional system, Quadrotor2, \toolname outperforms RL with lower tracking error and lower variance. Among these models, the SoS-based method cannot handle Cart-pole since it does not satisfy the sparse assumption for the control matrix $B$ and TLPRA since the dynamics is too complicated and hard to be approximated with low-degree polynomials. We did not evaluate the SoS-based method on other models since the case-by-case hyper-parameters tuning is too time consuming without domain expertise on the models. For Cart-pole, Pendulum, and TLPRA, the MPC-PyTorch library had numerical errors. 

\begin{figure}[t]
\centering
    \includegraphics[width=0.45\textwidth,trim=0 0 0 0,clip]{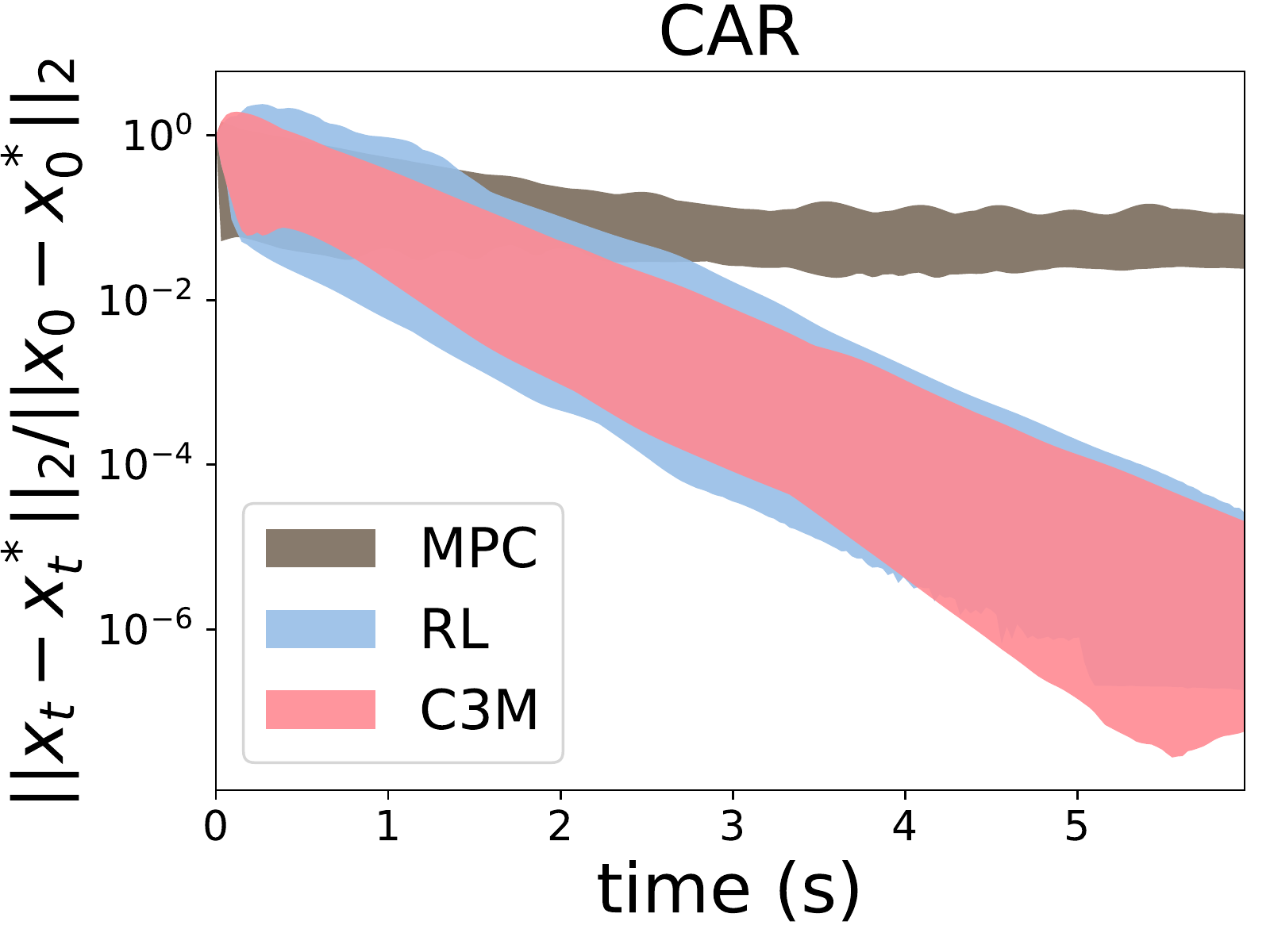}
    \includegraphics[width=0.45\textwidth,trim=0 0 0 0,clip]{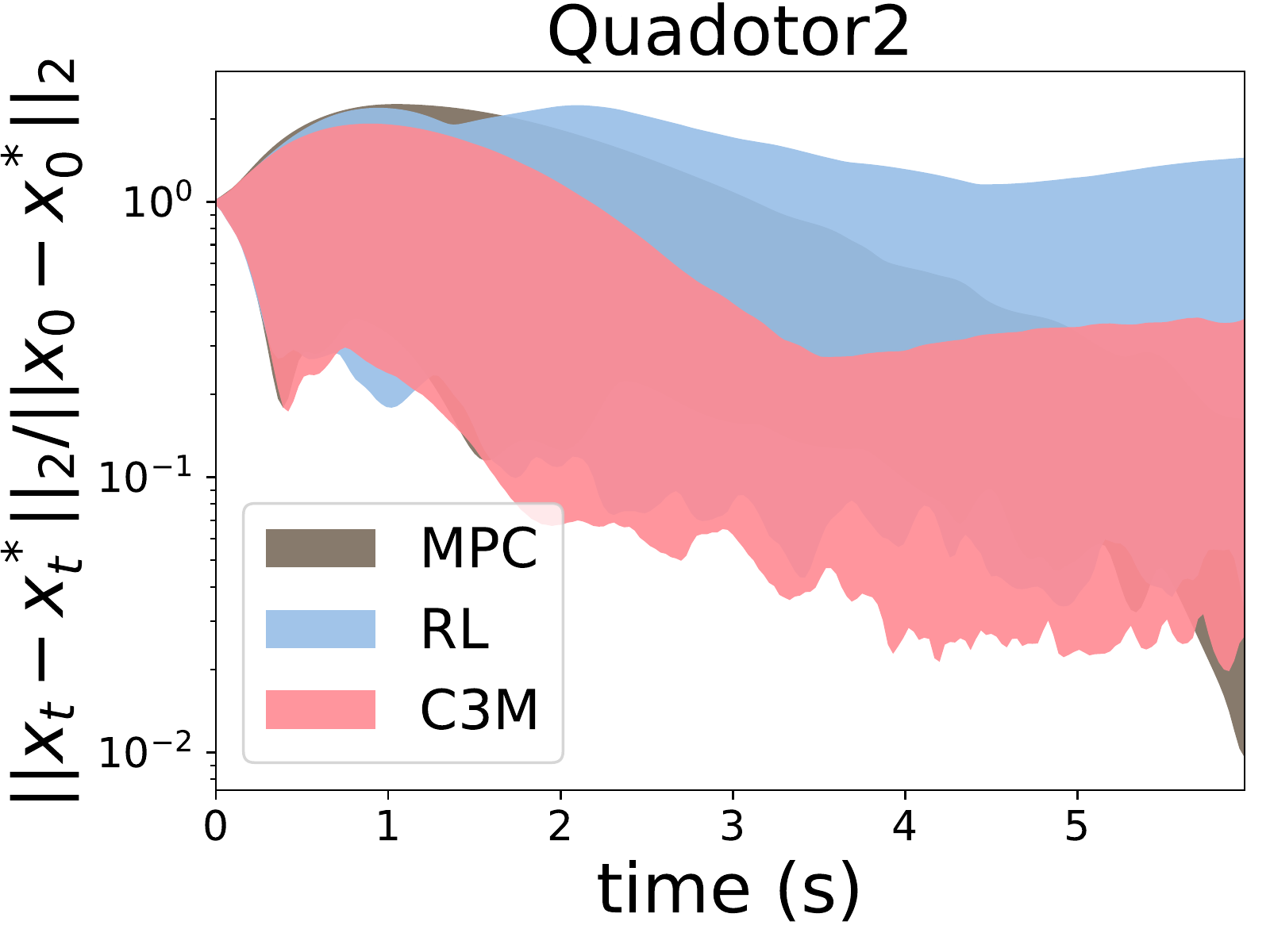}\\
    \includegraphics[width=0.45\textwidth,trim=0 0 0 0,clip]{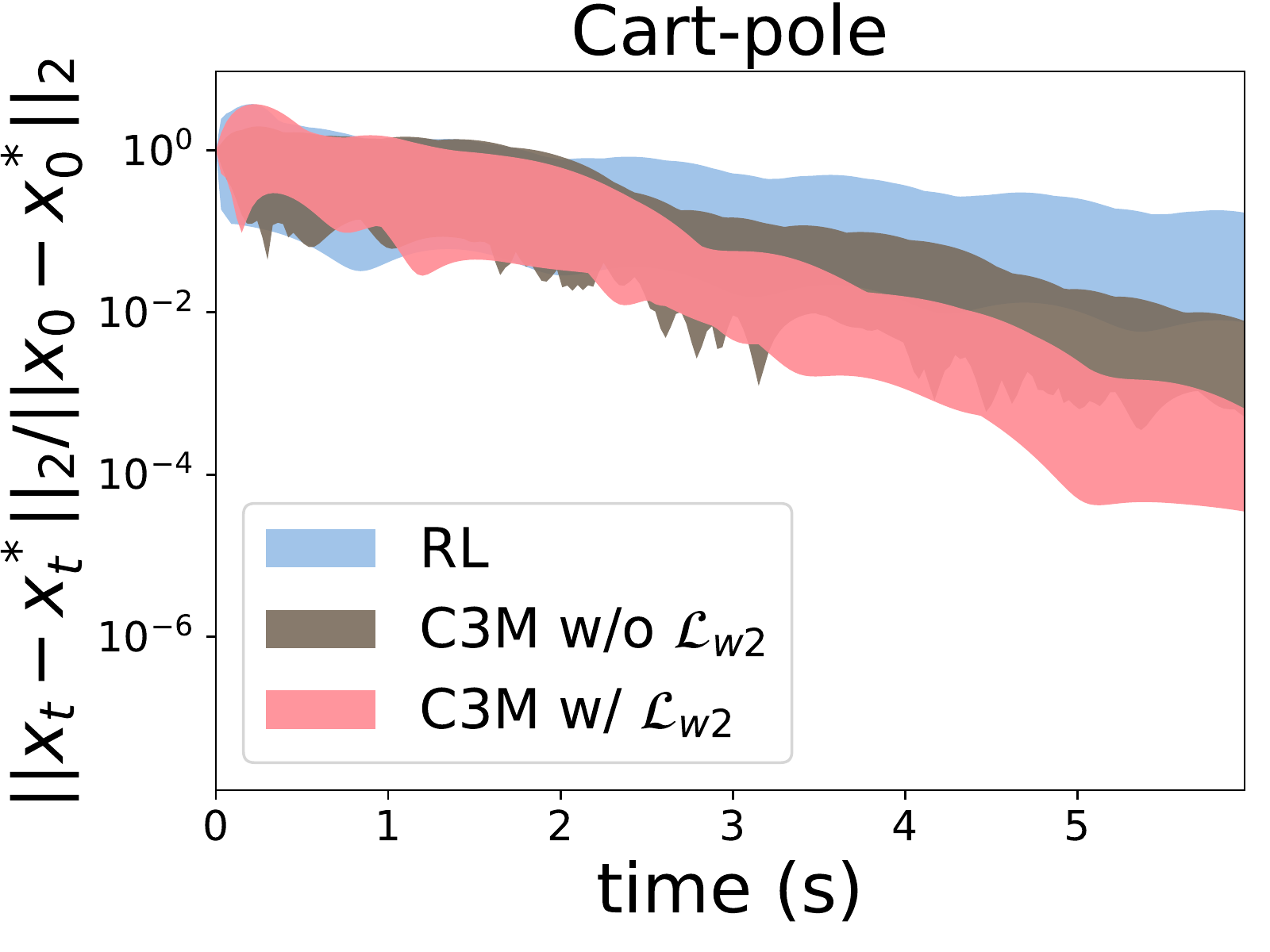}
    \includegraphics[width=0.45\textwidth,trim=0 0 0 0,clip]{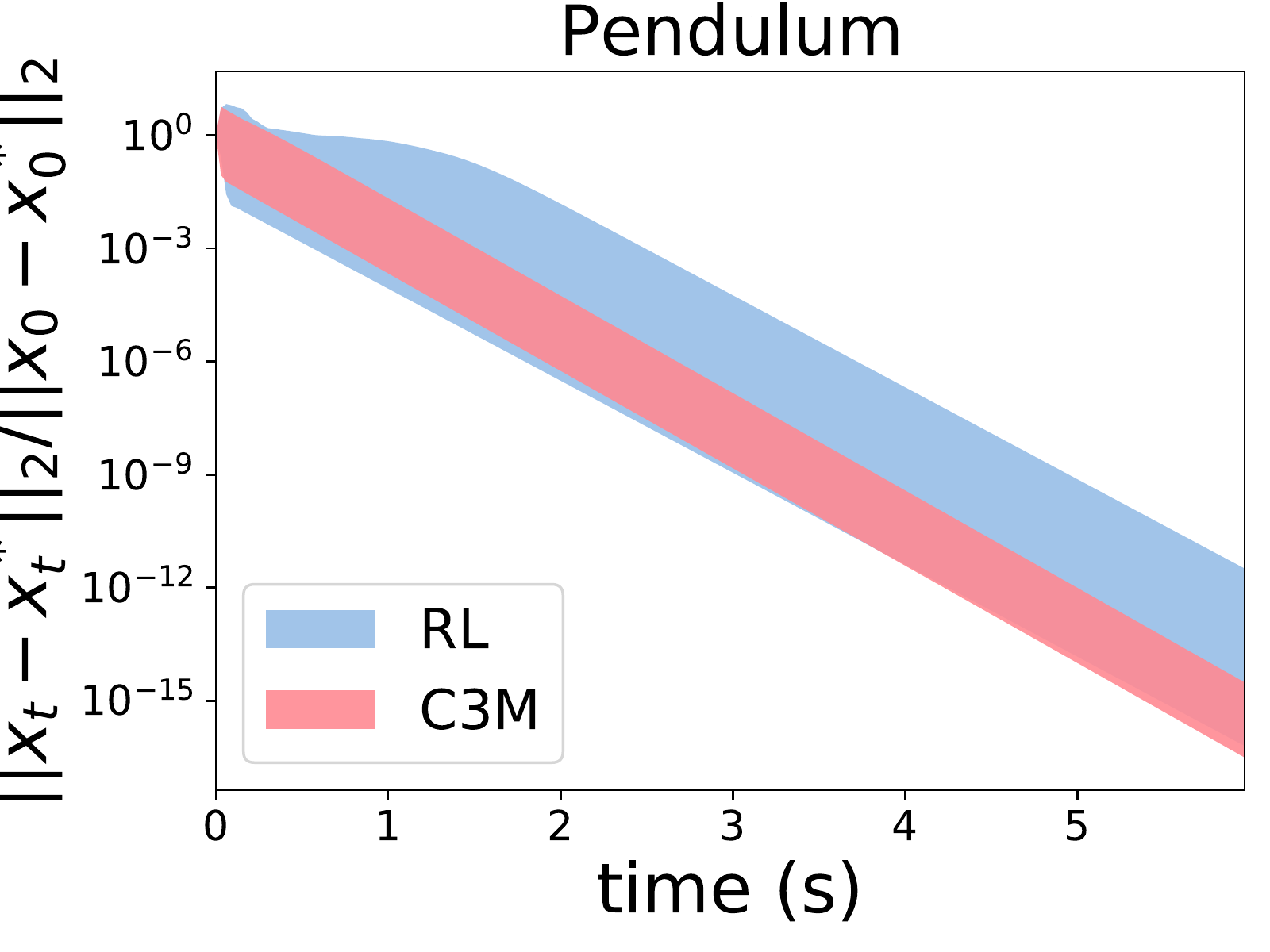}\\
    \includegraphics[width=0.45\textwidth,trim=0 0 0 0,clip]{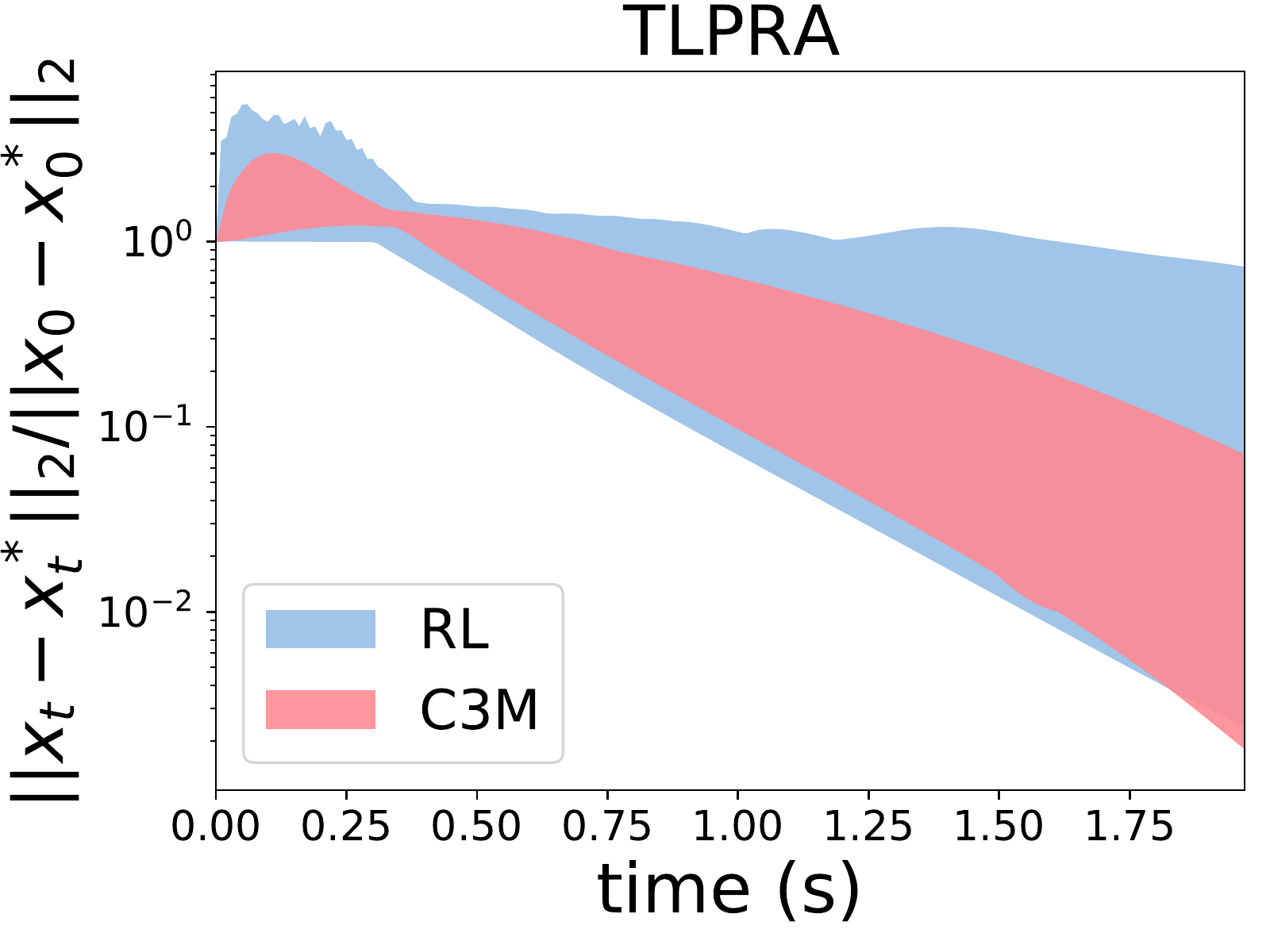}
\caption{\footnotesize Normalized tracking error on the benchmarks using different approaches. The $y$ axes are in log-scale. The tubes are tracking errors between mean plus and minus one standard deviation over 100 trajectories.}
\label{fig:results_more}
\end{figure}

\begin{table}[h!]
    \centering
    \resizebox{\columnwidth}{!}{
    \begin{tabular}{c|c|c|c|c|c|c|c|c|c|c|c}
         \toprule
         \multirow{2}{*}{Model} & \multicolumn{2}{c|}{Dim} & \multicolumn{3}{c|}{Execution time per step (ms)} & \multicolumn{3}{c|}{Tracking error (AUC)} &  \multicolumn{3}{c}{Convergence rate ($\lambda/C$)}\\
         \cline{2-12}
         & $n$ & $m$ & \toolname & MPC & RL & \toolname & MPC & RL & \toolname & MPC & RL\\
         \hline
         CAR & 4 & 2 & 0.46 & 2300 & 0.46 & 0.290 & 0.509 & 0.275 & 1.970/2.69 & 0.493/1.49 & 1.836/4.99\\
         Quadrotor2 & 10 & 3 & 0.45 & 7460 & 0.45 & 1.239 & 1.673 & 2.081 & 0.400/3.52 & 0.451/3.98 & 0.237/3.78\\
         Cart-pole & 4 & 1 & 0.45 & - & 0.45 & 0.865 & - & 1.207 & 1.100/4.85 & - & 0.509/4.10\\
         Pendulum & 2 & 1 & 0.37 & - & 0.37 & 0.091 & - & 0.092 & 5.772/6.82 & - & 2.675/9.46\\
         TLPRA & 4 & 2 & 0.42 & - & 0.42 & 1.082 & - & 1.086 & 2.035/4.91 & - & 1.883/6.82\\
         \bottomrule
    \end{tabular}
    }
    \caption{\footnotesize Comparison results of \toolname vs. other methods. Here, $n$ and $m$ are dimensions of the state and input space. The MPC library used in experiments failed to handle some models due to numerical issues (we got some \textbf{nan} errors).}
    \label{tab:results_more}
\end{table}
\subsection{Sampled trajectories under disturbance}
We show some plots of actual trajectories of the controlled systems in Fig.~\ref{fig:sampled_traj}. We simulate the controlled Dubin's car and Quadrotor2 system using different control strategy and with different levels of disturbances. In order to model a more realistic disturbance (e.g. wind), we do not adopt white noise as disturbance. Instead, $d(t)$ is a piece-wise constant function, where the length of each constant interval and the magnitude on each interval are uniformly randomly sampled from $[0, 1]$ (seconds) and $[0, \sigma]$ respectively. Since the controlled system is proved to be contracting using C3M, the tracking error is still bounded even under additive random disturbances. We can also see from Fig.~\ref{fig:sampled_traj} that even with disturbance, the trajectories of the closed-loop system using \toolname has less tracking error than the system with controllers using RL or MPC.

\begin{figure}[t]
\centering
    \includegraphics[width=0.32\textwidth,trim=0 0 0 0,clip]{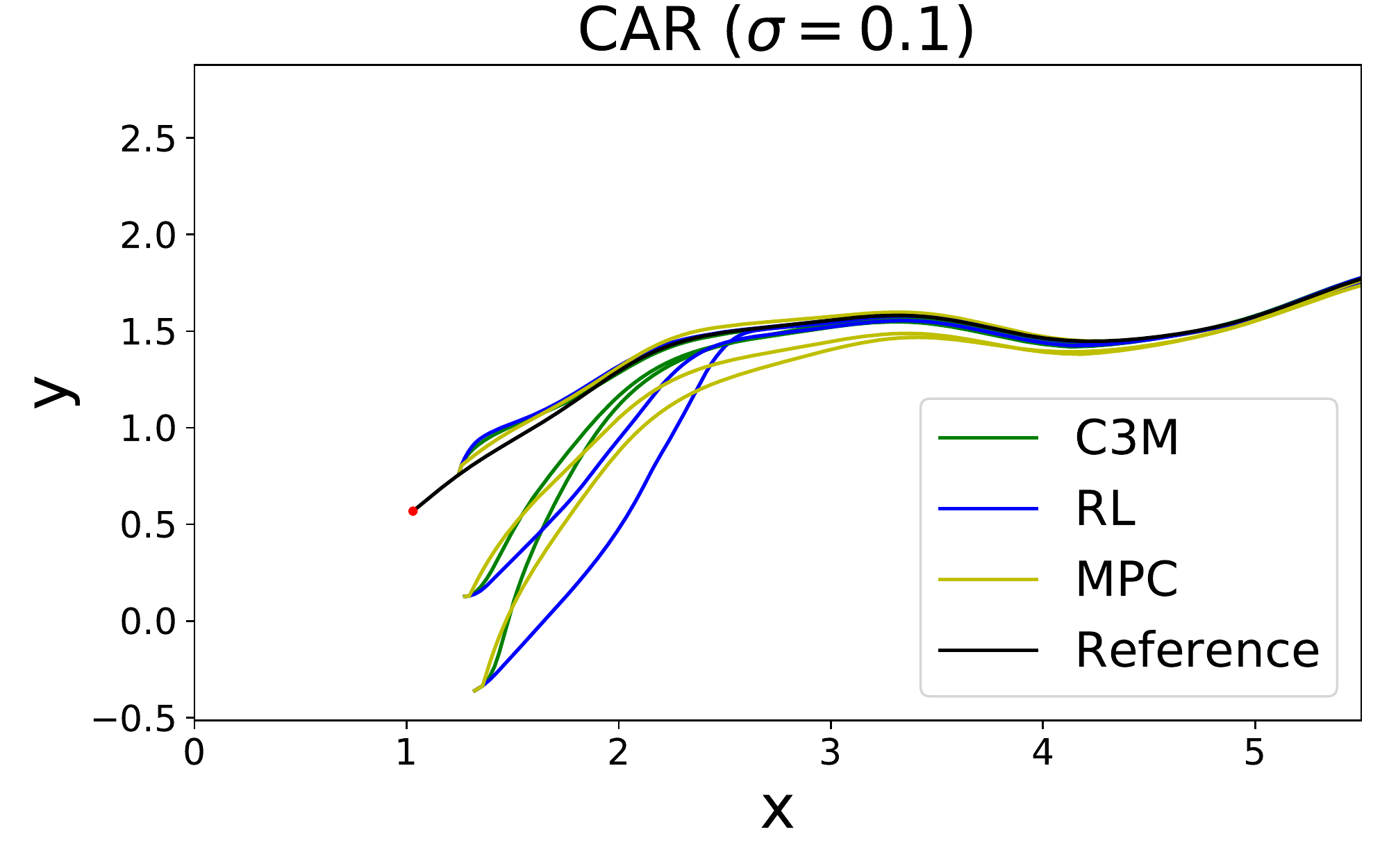} \includegraphics[width=0.32\textwidth,trim=0 0 0 0,clip]{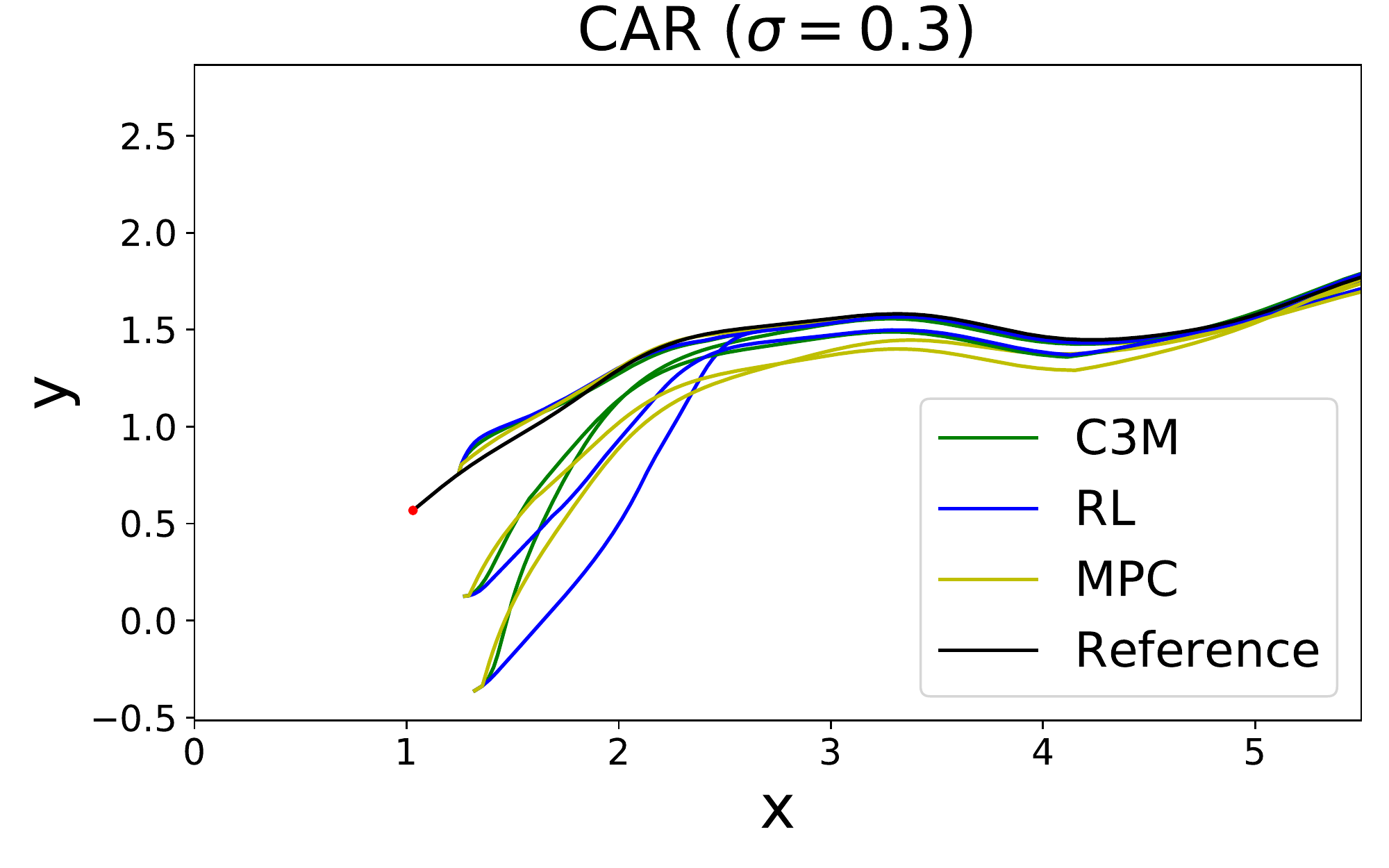} \includegraphics[width=0.32\textwidth,trim=0 0 0 0,clip]{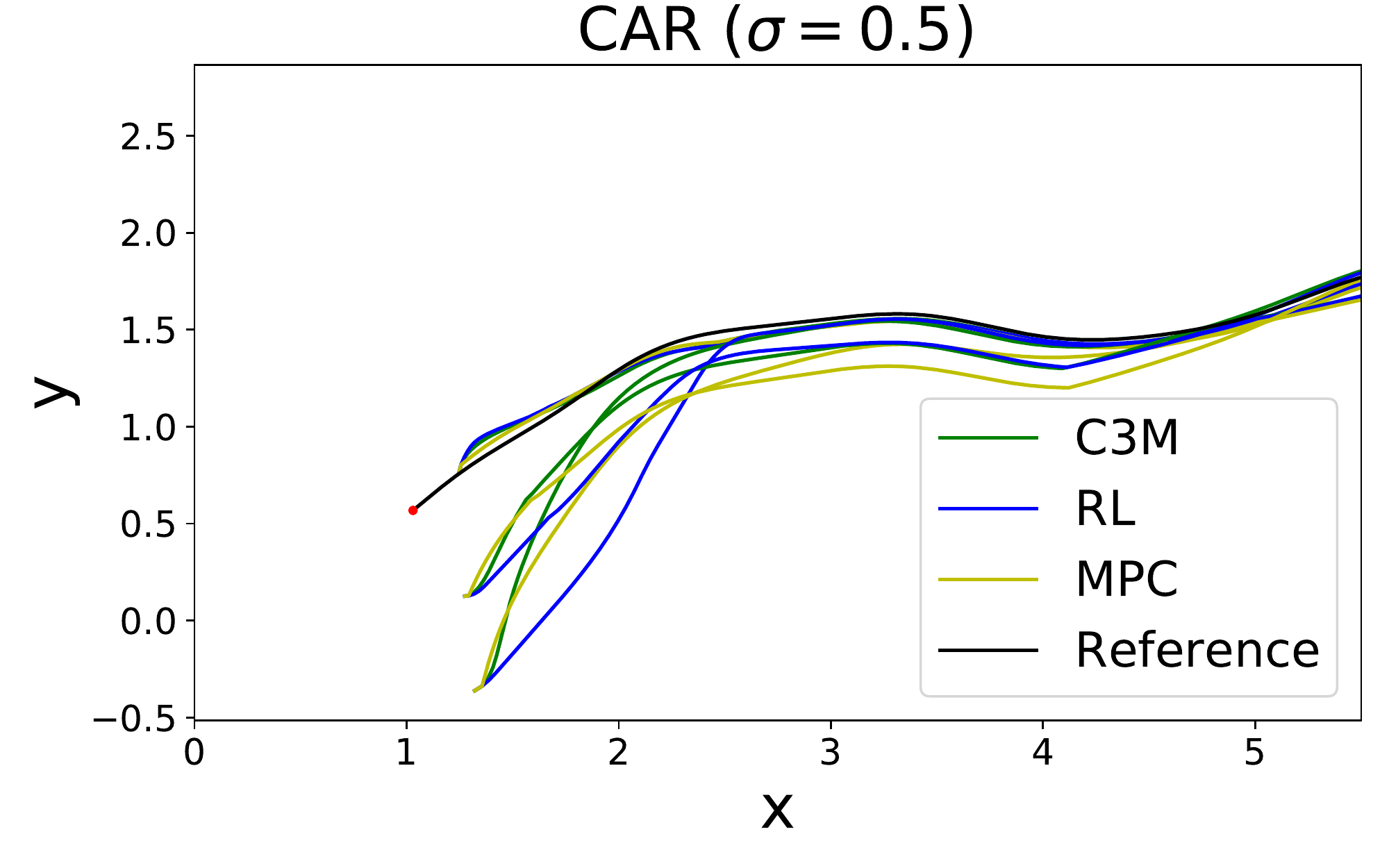}\\
    \includegraphics[width=0.32\textwidth,trim=0 0 0 0,clip]{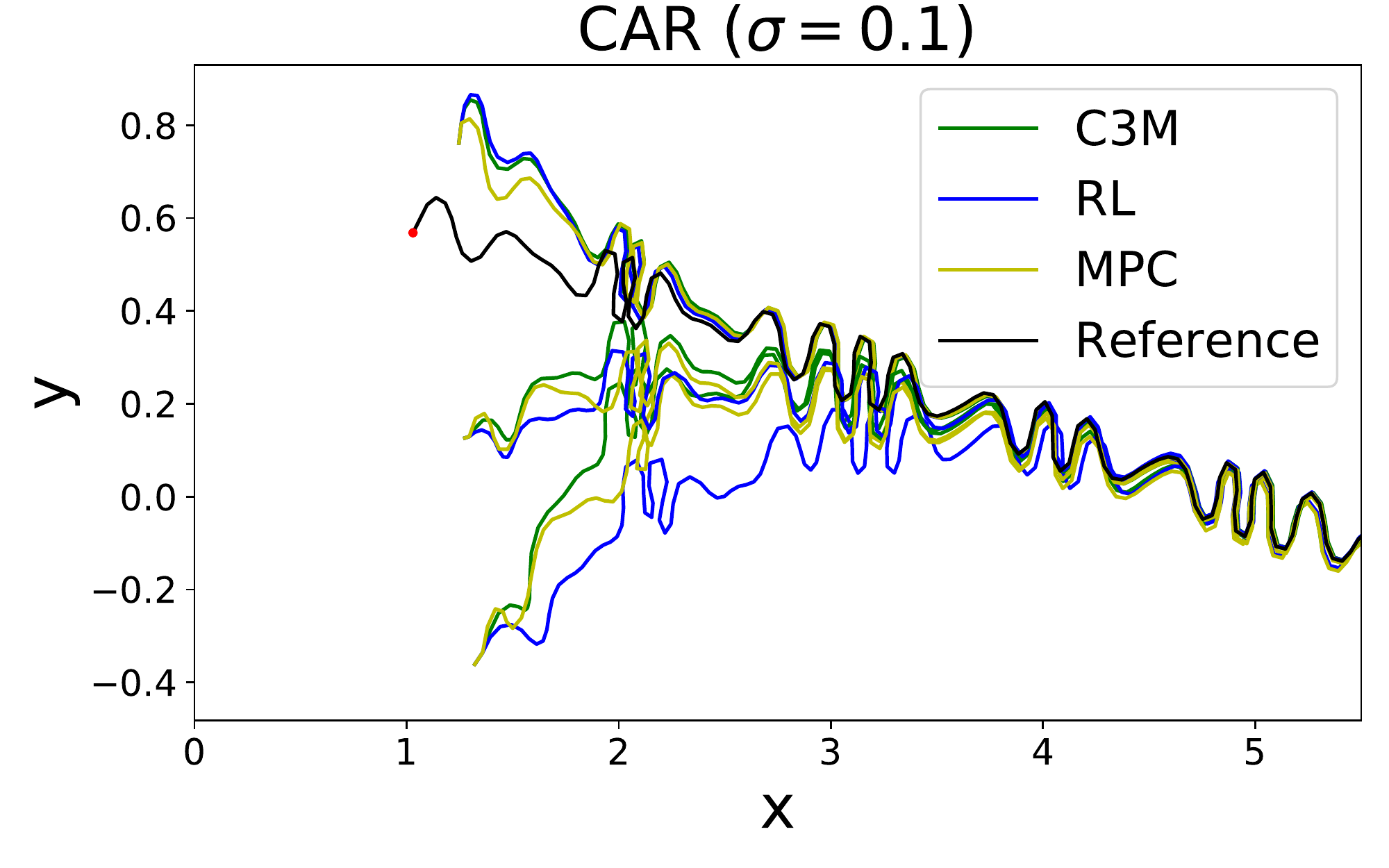} \includegraphics[width=0.32\textwidth,trim=0 0 0 0,clip]{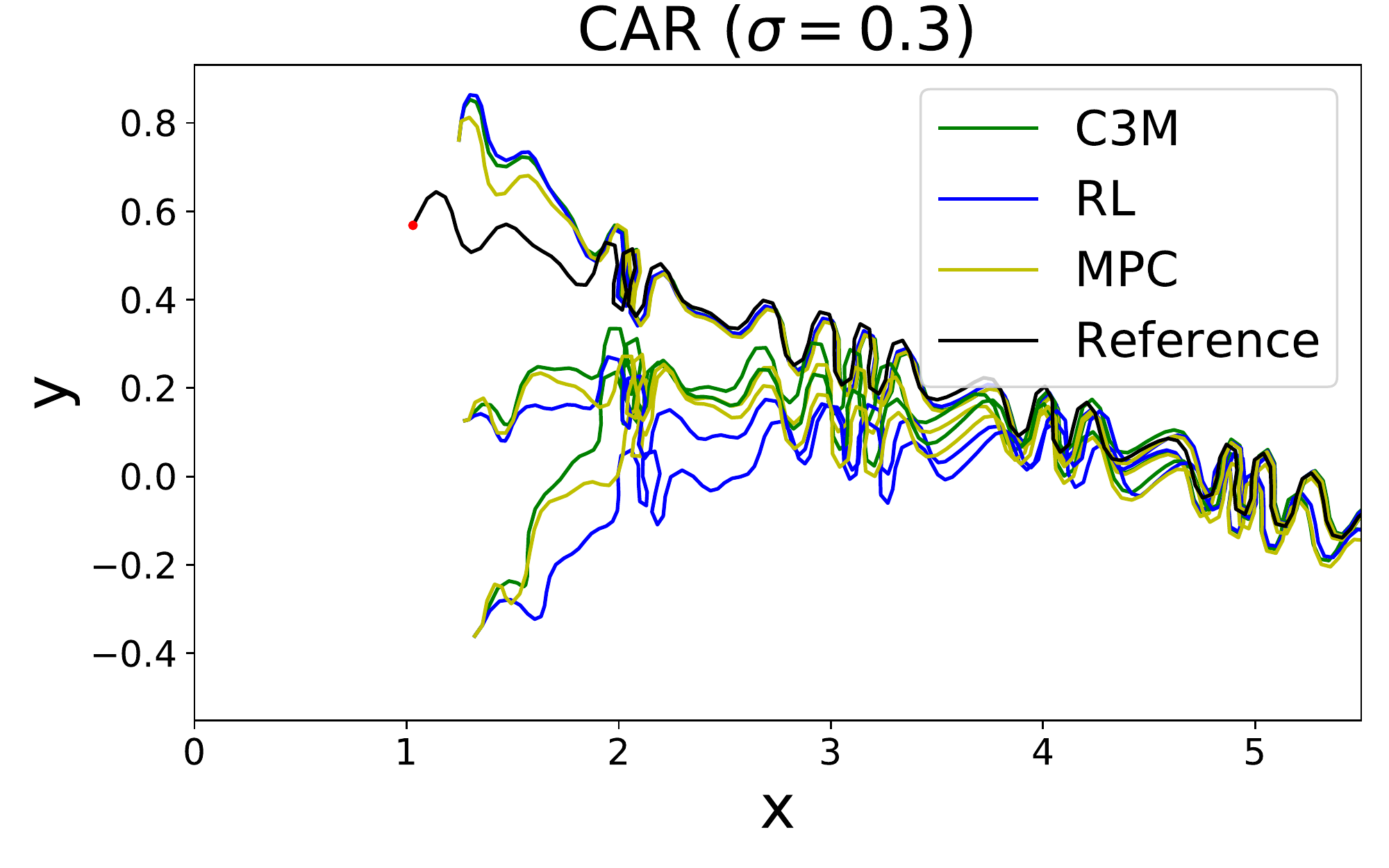} \includegraphics[width=0.32\textwidth,trim=0 0 0 0,clip]{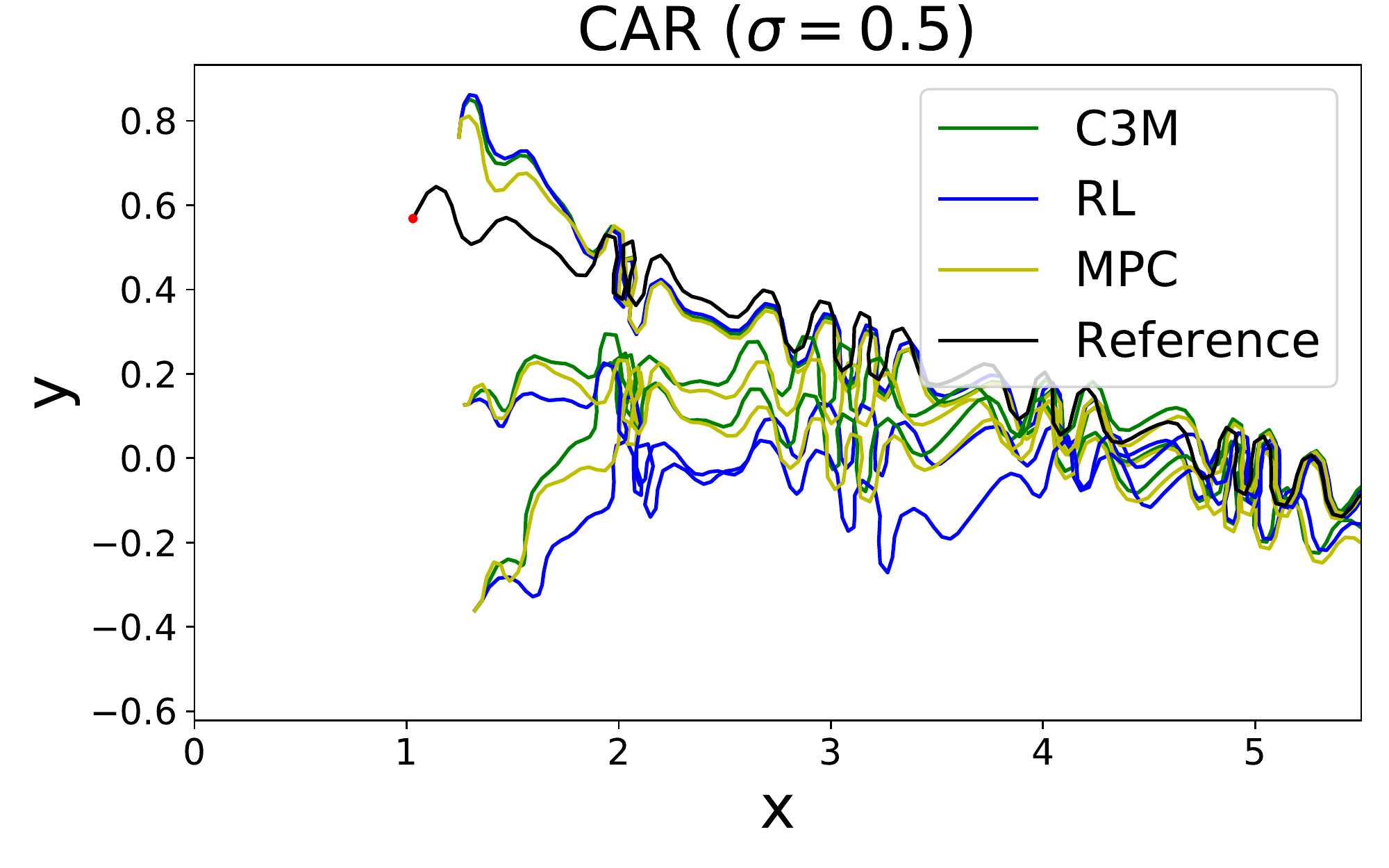}\\
    \includegraphics[width=0.32\textwidth,trim=0 0 0 0,clip]{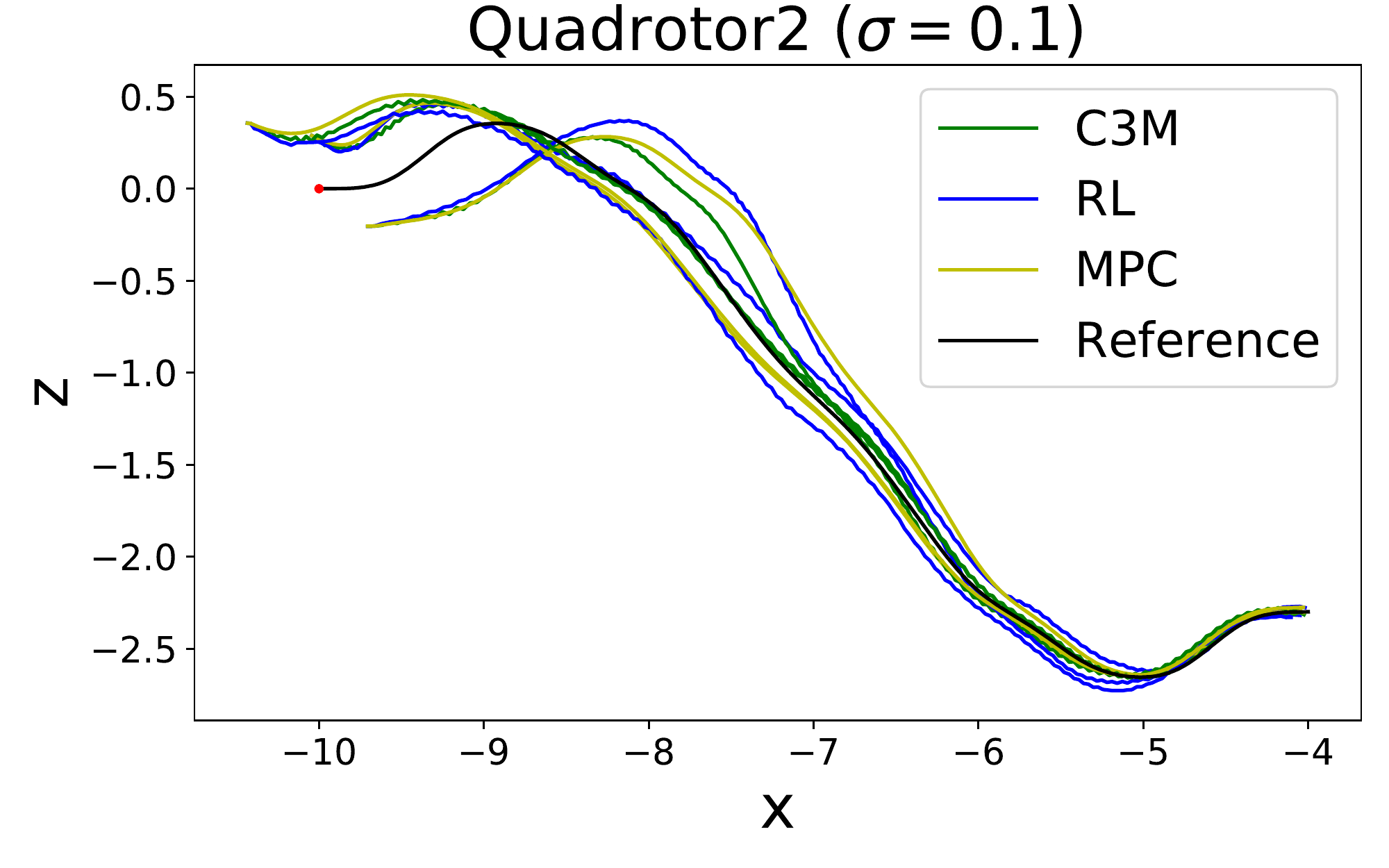} \includegraphics[width=0.32\textwidth,trim=0 0 0 0,clip]{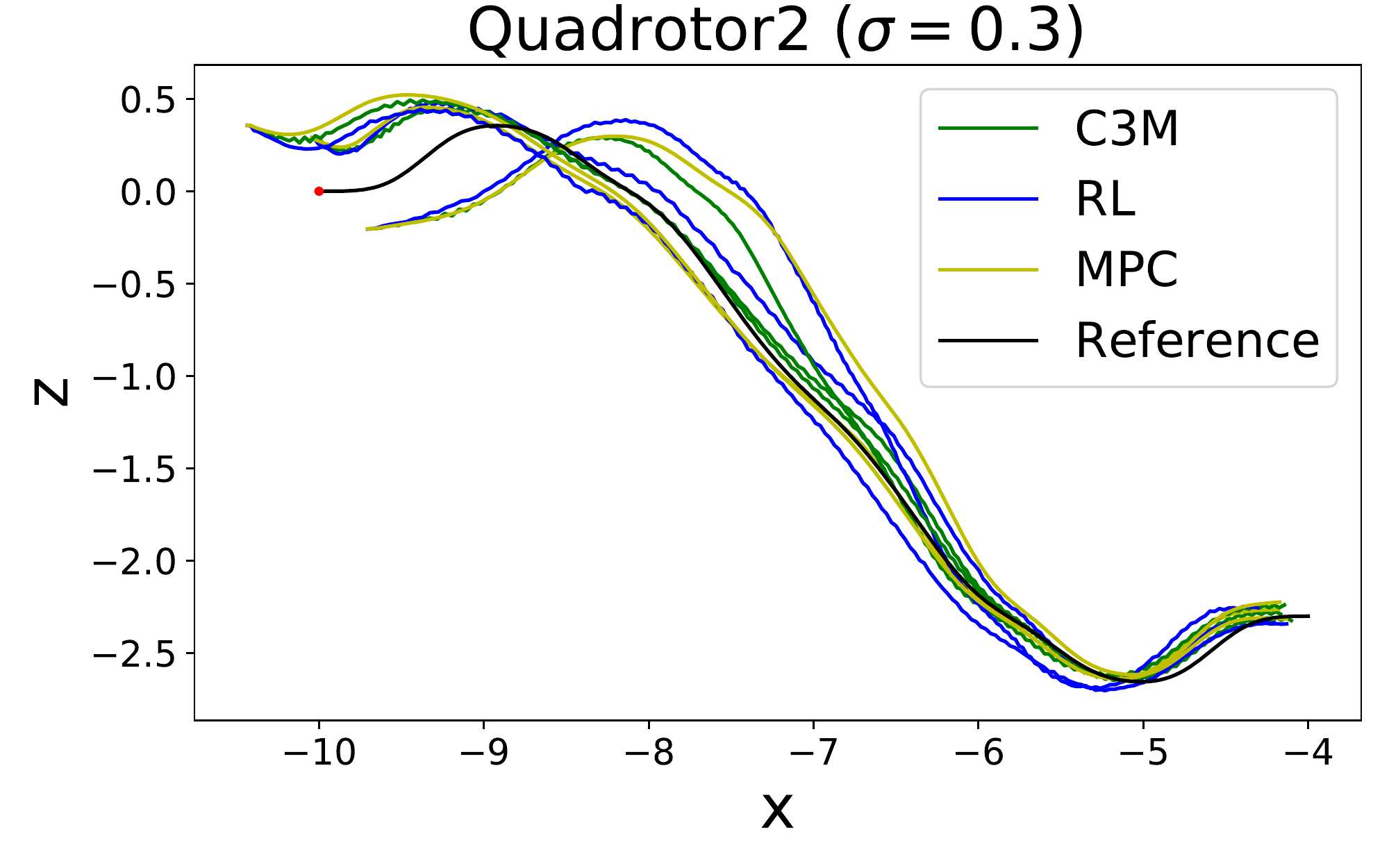} \includegraphics[width=0.32\textwidth,trim=0 0 0 0,clip]{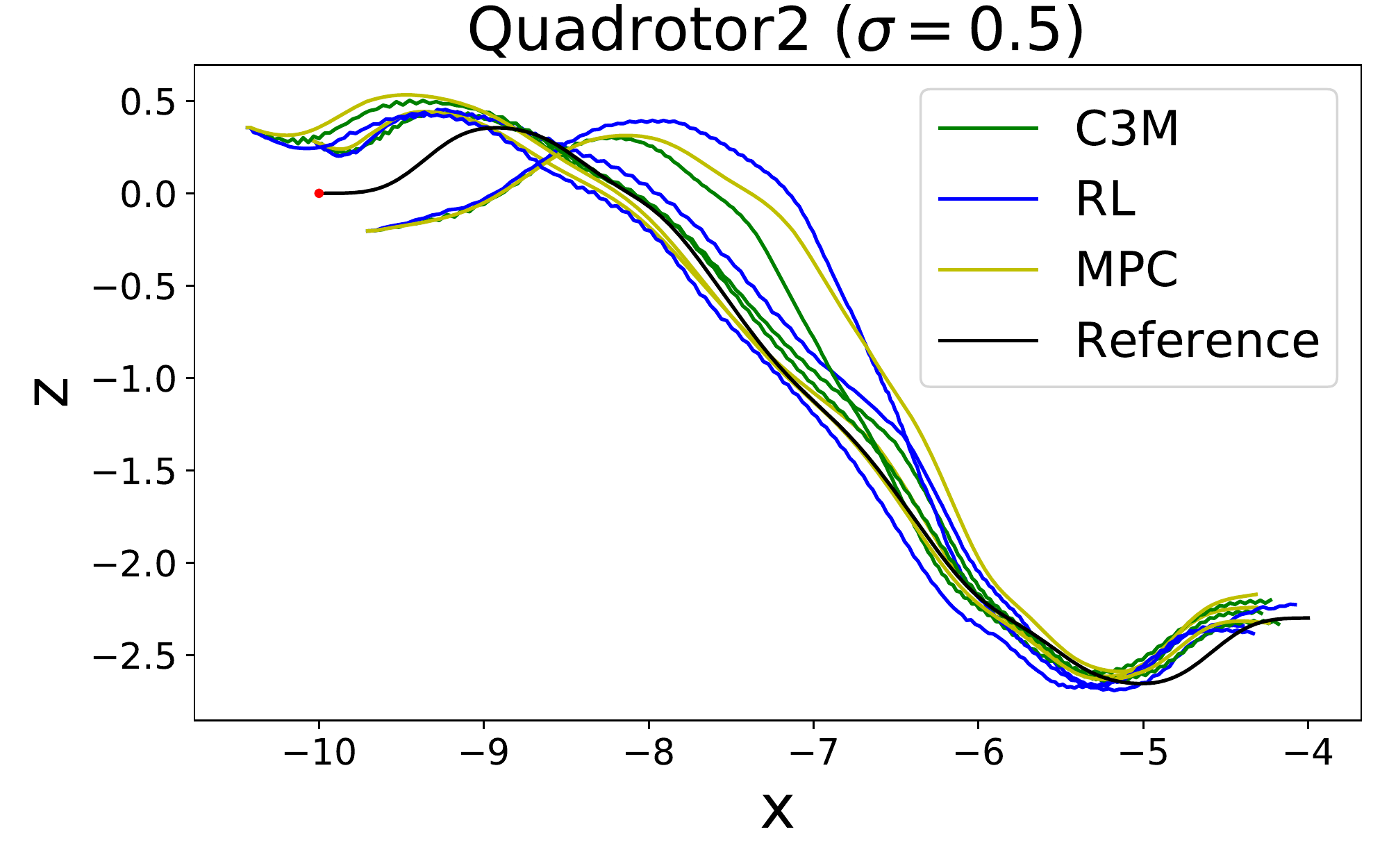}
\caption{\footnotesize Actual trajectories of the controlled systems in the presence of disturbances. Red dots indicate the starting points of the reference trajectories. Moreover, in order to show that the proposed method can work for any reference trajectories, in the second row, we use a reference control $u^*$ that consists of a very high-frequency component.}
\label{fig:sampled_traj}
\end{figure}

\section{More discussions on related work}
In the broad area of safe motion planning, driving the robots to some waypoints or follow some nominal trajectory without hitting obstacles in dynamically changing environments is a very challenging task, especially when robotics systems are nonlinear and nonholonomic. In this section, we provide more details on some related techniques in this area.

Sample-based planning methods including rapidly-exploring random trees (RRT)~\citep{lavalle1998rapidly}, probabilistic road maps (PRM)~\citep{kavraki1998analysis}, fast marching tree (FMT)~\citep{janson2015fast}, and many others~\citep{richter2016polynomial, karaman2011anytime, kobilarov2012cross}, plan safe open-loop trajectories by exploring the environment using samples. However, it is difficult to use sample-based planning to handle uncertainty and disturbances with safety guarantees.
In order to plan safe trajectories in the presence of disturbances, the uncertainty should be handled appropriately. The disturbances are usually assumed to be bounded, and the safety can be guaranteed by dealing with the worst case safely.

One may model safe motion planning problems as differential games, where the disturbances are modeled as an adversarial agent. Backward reachable sets and the corresponding optimal controller can be computed by solving Hamilton-Jacobi PDE using level-set methods~\citep{tomlin2003computational, gillula2010design, huang2011differential, herbert2017fastrack, fridovich2018planning, chen2018decomposition}. However, this class of methods suffers from poor scalablity. Solving Hamilton-Jacobi PDE for high-dimensional models is computationally expensive. Moreover, such high computational complexity limits their capability of solving online planning problems for dynamical environments. In comparison, our learned tracking controller can be executed in sub-millisecond level and therefore supports online planning. Moreover, because of the learned contraction metric, we can pre-compute the tracking error bound, which can be further used to understand how far the nominal trajectories should be from the obstacles. 

Barrier functions are another widely-used class of certificates for safety. A barrier function is a function of state whose time-derivative on the zero-level set is uniformly negative, which makes the zero-level set an invariant. In~\citep{prajna2007framework, barry2012safety}, barrier functions are computed for given controllers and any admissible disturbances and are used for verification of the closed-loop system. As an extension to controlled systems, control barrier functions~\citep{ames2014control, ames2019control, taylor2019learning} enable synthesis of safety-guaranteed controllers and thus have gained increasing popularity. In~\citep{ames2014control}, the authors jointly leveraged control barrier functions (CBF) and control Lyapunov functions (CLF) as safety and objective certificates respectively for control synthesis in the context of adaptive cruise control and lane keeping. However, finding valid CBF for general models is hard. Recently, the advances of machine learning have been utilized to help search for CBF~\citep{taylor2019learning}.

The idea of synthesizing a controller and bounding the tracking error meanwhile through pre-computation has gained increasing popularity. It enables safe motion planning for dynamically changing environments. Control Lyapunov functions can be used for this purpose~\cite{ravanbakhsh2016robust}. Learning-based CLF has also been studied in some recent works~\citep{choi2020reinforcement,khansari2014learning, robey2020learning}. However, in order to use CLF for synthesizing controller to track reference or nominal trajectories instead to driving the system to a fixed goal set, the original dynamics has to be converted into error dynamics depending on the reference trajectory, which limits the diversity of available reference trajectories. Similarly, funnel library methods have been studied in~\citep{majumdar2013robust, majumdar2017funnel}, where tracking controllers for a fixed set of reference trajectories and the corresponding tracking error are computed offline. In the online planning phase, the safety of the reference trajectories are examined using the funnels corresponding to them, and the safest one is chosen. Again, this class of methods lack diversity of available reference trajectories which is critical to safe motion planning in complicated environments. LQR-Trees~\cite{tedrake2009lqr} builds a tree of LQR controllers backward from the goal set such that the total contraction regions of the nodes cover the whole state space, however, it cannot handle scenarios where the environments are unknown until runtime. Tube Model Predicative Control (TMPC) is another class of related techniques, where a tracking controller is computed such that the actual trajectory remains in a tube centered at the planned MPC nominal trajectory in the presence of bounded disturbances. TMPC techniques have been studied for linear system in many works~\citep{langson2004robust, mayne2005robust, farina2012tube}. As for non-linear systems, linearization and Lipschitz-continuity-based reachibility analysis are used~\citep{kogel2015discrete, yu2013tube, mayne2011tube}. However, in the latter case, the tube are often  too conservative to be used in motion planning with limited freespace.

Contraction analysis~\citep{lohmiller1998contraction} is another series of methods for analyzing the incremental stability of systems. Recently, it has been extended to controlled system in~\citep{manchester2017control} and thus enabled tracking control synthesis for arbitrary reference trajectories with guaranteed bound on tracking error. The most challenging part in contraction analysis is the search for a valid contraction metric which entails solving Linear Matrix Inequalities (LMIs). In~\citep{aylward2008stability}, the authors proposed to solve this feasibility problem with Sum-of-Squares (SoS) programming. In~\citep{singh2019robust}, the authors extended the SoS-based method to controlled systems, i.e. search for a \textit{control} contraction metric (CCM) instead of an ordinary one, and proposed a more general method for synthesizing control given a valid CCM. However, in order to apply SoS-based methods, the dynamics of the system has to be represented by polynomials or can be approximated by polynomials. Furthermore, the method proposed in~\citep{singh2019robust} relies on an assumption on the structure which encodes the controllablity of the system. Different from ours, the method proposed in~\citep{singh2019robust} searches for a metric using the structural assumption and synthesize the controller
using geodesics, which entails heavy computation. Moreover, geodesics cannot be computed exactly, and thus the authors in~\citep{singh2019robust} resorted to the optimization-based method proposed in~\citep{leung2017nonlinear} as an approximation of the geodesics.

Combining machine learning and contraction theory has also been studied in some recent work. In~\citep{singh2019learning}, the authors proposed a learning-based method for identifying an unknown system with an additional constraint that such learned dynamics admit a valid CCM. In~\citep{tsukamoto2020neural}, the authors proposed a control synthesis method for nonlinear systems where the dynamics is written as a convex combination of multiple state-dependent coefficients (i.e. $f(x,t)$ written as $A(x,t)x$) and made use of RNNs to model a time-varying metric. Again, different from our method, the controller is not jointly learned with the metric. In contrast, our approach can simultaneously synthesize the controller and the CCM certificate without any additional structural assumptions for the system.

\end{document}